\documentclass{article}
\PassOptionsToPackage{numbers, compress}{natbib}

\usepackage[preprint]{neurips_2026}

\usepackage[utf8]{inputenc} 
\usepackage[T1]{fontenc}    
\usepackage{url}         
\usepackage{booktabs}       
\usepackage{amsfonts}    
\usepackage{nicefrac}      
\usepackage{microtype}      
\usepackage{xcolor}       

\usepackage{microtype}
\usepackage{graphicx}
\usepackage{subcaption}
\usepackage{booktabs} 

\usepackage[utf8]{inputenc}
\usepackage{url}          
\usepackage{amsfonts}       
\usepackage{nicefrac}    
\usepackage[dvipsnames]{xcolor}   
\usepackage[backref=page]{hyperref}
\usepackage{wrapfig,lipsum,booktabs}
\usepackage{booktabs}
\usepackage{multirow}
\usepackage{url}
\usepackage{tikz}
\usetikzlibrary{arrows}
\usepackage{paralist}
\usepackage[stable]{footmisc}
\usepackage[textwidth=0.7in, textsize=tiny]{todonotes}
\usepackage{ifthen}
\usepackage{rotating}
\usepackage{forest}
\usepackage{xparse}
\usepackage[normalem]{ulem}
\usepackage{bbm}
\usepackage{natbib}

\usepackage{enumitem}
\usepackage{adjustbox}
\usepackage{diagbox}

\usepackage{bm}
\usepackage{amsmath} 
\usepackage{amsthm}
\usepackage{amsfonts}
\usepackage{amssymb}
\usepackage{mathtools}
\usepackage{fixmath}
\usepackage{siunitx}
\usepackage{xfrac}

\usepackage{arydshln}

\usepackage{thmtools}		
\usepackage{mleftright}
\usepackage{stmaryrd}
\usepackage{nicefrac}
\usepackage{fontawesome5}
\usepackage{graphicx}

\usepackage{hyperref}   
\usepackage[capitalize,noabbrev]{cleveref}

\theoremstyle{plain}
\newtheorem{theorem}{Theorem}[section]
\newtheorem{proposition}[theorem]{Proposition}
\newtheorem{lemma}[theorem]{Lemma}

\theoremstyle{definition}

\theoremstyle{remark}

\usepackage{appendix}
\usepackage{caption}

\usepackage{algorithm}
\usepackage{algpseudocode}

\newcommand{\hb}{\mathbold{h}}

\newcommand{\UPD}{\mathsf{UPD}}
\newcommand{\AGG}{\mathsf{AGG}}

\newcommand{\new}[1]{\emph{#1}}

\newcommand{\trans}{^\intercal}
\renewcommand{\vec}[1]{\mathbold{#1}}

\newcommand{\xhdr}[1]{{\noindent\bfseries #1}}

\usepackage{bm}

\definecolor{cb-black}      {RGB}{  0,   0,   0}
\definecolor{cb-blue-green} {RGB}{  0,  073,  073}
\definecolor{cb-rose}       {RGB}{255, 109, 182}
\definecolor{cb-salmon-pink}{RGB}{255, 182, 119}
\definecolor{cb-purple}     {RGB}{ 73,   0, 146}
\definecolor{cb-blue}       {RGB}{ 0, 109, 219}
\definecolor{cb-lilac}      {RGB}{182, 109, 255}
\definecolor{cb-blue-sky}   {RGB}{109, 182, 255}
\definecolor{cb-blue-light} {RGB}{182, 219, 255}
\definecolor{cb-burgundy}   {RGB}{146,   0,   0}
\definecolor{cb-brown}      {RGB}{146,  73,   0}
\definecolor{cb-clay}       {RGB}{219, 209,   0}
\definecolor{cb-green-lime} {RGB}{ 36, 255,  36}
\definecolor{cb-yellow}     {RGB}{255, 255, 109}

\definecolor{cb-green-sea}{HTML}{0173B2}
\definecolor{cb-burgundy}{HTML}{029E73}
\definecolor{cb-lilac}{HTML}{D55E00}

\definecolor{first}{HTML}{029E73}
\definecolor{second}{HTML}{0173B2}
\definecolor{third}{HTML}{D55E00}
\definecolor{fourth}{HTML}{8000ff}

\definecolor{sns_cb1}{HTML}{029E73}
\definecolor{sns_cb2}{HTML}{0173B2}
\definecolor{sns_cb3}{HTML}{D55E00}
\definecolor{sns_cb4}{HTML}{CC78BC}
\definecolor{sns_cb5}{HTML}{ECE133}
\definecolor{sns_cb6}{HTML}{56B4E9}

\newcommand{\WLk}[1]{#1\text{-}\mathsf{WL}}
\newcommand{\FWLk}[1]{#1\text{-}\mathsf{FWL}}

\newcommand{\vctfwl}{\mathsf{VC}\text{-2-}\mathsf{FWL}}
\newcommand{\mcwl}{\mathsf{MC}\text{-}\mathsf{WL}}
\newcommand{\mcmpnn}{\mathsf{MC}\text{-}\mathsf{MPNN}}

\newcommand{\dmcwl}{\delta \text{-} \mathsf{MC} \text{-}\mathsf{WL}}
\newcommand{\dmpnn}{\delta \text{-} \mathsf{MC} \text{-}\mathsf{MPNN}}

\newcommand{\ip}[2]{\left\langle #1,\,#2\right\rangle}

\newcommand{\Proj}{\mathrm{Proj}}
\newcommand{\Acal}{\mathcal{A}}

\NewDocumentCommand{\dgal}{sO{}m}{%
  \IfBooleanTF{#1}
    {\dgalext{#3}}
    {\dgalx[#2]{#3}}%
}

\NewDocumentCommand{\dgalx}{om}{%
  \sbox0{\mathsurround=0pt$#1\{$}%
  \sbox2{\{}%
  \ifdim\ht0=\ht2
    \{\kern-.45\wd2 \{#2\}\kern-.45\wd2 \}%
  \else
    \mathopen{#1\{\kern-.5\wd0 #1\{}
    #2
    \mathclose{#1\}\kern-.5\wd0 #1\}}
  \fi
}

\title{Solving Max-Cut to Global Optimality via Feasibility-Preserving Graph Neural Networks}

\author{%
Hao Chen$^{1*}$ \quad
Chendi Qian$^{2*}$ \quad
Christopher Morris$^2$ \quad
Andrea Lodi$^3$ \quad
Can Li$^1$ \\
$^1$Davidson School of Chemical Engineering, Purdue University \\
$^2$Faculty of Computer Science, RWTH Aachen University \\
$^3$Jacobs Technion-Cornell Institute, Cornell Tech \\
\texttt{\{chen4433, canli\}@purdue.edu} \\
\texttt{\{chendi.qian@log, morris@cs\}.rwth-aachen.de} \\
\texttt{andrea.lodi@cornell.edu} \\
$^*$Equal contribution.
}

\begin{document}

\maketitle

\begin{abstract}
Exact solution of hard combinatorial optimization problems often relies on strong convex relaxations, but solving these relaxations repeatedly inside a branch-and-bound algorithm can be prohibitively expensive. Hence, we consider this challenge for \new{Max-Cut problem} (Max-Cut), where branch and bound commonly uses semidefinite programming (SDP) relaxations to bound subproblems. We propose a Max-Cut-specific \new{graph neural network} that serves as a principled, lightweight neural proxy for these SDP solvers and can be plugged directly into an exact branch-and-bound framework. The proposed architecture has update steps of complexity $\mathcal{O}(n^2 + ne)$, and predicts both primal- and dual-feasible SDP solutions. The primal SDP solutions yield feasible Max-Cut solutions via the Goemans--Williamson algorithm. In addition, it is trained in a self-supervised fashion without requiring solved SDP relaxations as labels. Empirically, we show that our architecture can substantially reduce the cost of bounding in exact Max-Cut solving by up to $10.6 \times$ compared with using the state-of-the-art SDP solver Mosek. Our work highlights the potential of learned, validity-preserving surrogates for accelerating exact optimization over structured convex relaxations.
\end{abstract}

\section{Introduction}

\new{Combinatorial optimization} (CO) is a field at the intersection of optimization, operations research, discrete mathematics, and computer science, with many high-impact real-world applications, including routing, scheduling, and network design~\citep{Kor+2012}. In practice, however, optimization instances are rarely arbitrary, i.e., they typically arise from structured, application-specific instance distributions. Classical exact solvers are largely distribution-agnostic and therefore do not explicitly exploit recurring structure in the instances they encounter~\citep{bengio2021machine}.

This observation has motivated a growing line of work on combining exact optimization with data-driven prediction. The central idea is to leverage machine learning to infer useful algorithmic information from previously seen instances while preserving the correctness guarantees of the underlying solver. In this context, \new{graph neural networks} (GNNs), and \new{message-passing neural networks} (MPNNs)~\citep{Sca+2009,Gil+2017} in particular, are a natural choice, i.e., they are permutation-equivariant, naturally represent variable-constraint interactions, and can process instances of varying size. Here, prior work has shown encouraging results for linear programming~\citep{chen2022representing,qian2024exploring} and mixed-integer linear programming, e.g., in branching~\citep{chen2024rethinking,Gasse2019,Nai+2020}, primal solution prediction~\citep{khalil2022mip}, and cut selection~\citep{Dez+2023,Li+2023,Paulus2022}; see~\citet{Cap+2023} for thorough survey.

Many practically important optimization problems, however, go beyond the linear setting. In particular, \new{semidefinite programming} (SDP) relaxations play a central role in hard discrete optimization problems, yet repeatedly solving SDPs inside an exact branch-and-bound algorithm remains computationally demanding \citep{vandenberghe1996semidefinite,wolkowicz2012handbook,wang2024tuningfree,han2025low,rendl2007branch}. Here, we study the \new{Max-Cut problem} (Max-Cut), a canonical \textsf{NP}-complete problem with numerous applications, e.g., in statistical physics and circuit layout design~\citep{barahona1988application,poljak1995maxcutsurvey}. Exact solution of Max-Cut can be achieved via branch and bound, where each search-tree node is bounded by solving the SDP relaxation of the corresponding subproblem. Although these SDP bounds are often strong, recomputing them throughout the search is typically the primary computational bottleneck, even with state-of-the-art SDP solvers such as Mosek~\citep{mosek}. 

\paragraph{Contributions} Hence, our goal is to replace the expensive SDP solvers with a learned neural proxy that remains compatible with exact optimization. Building on recent principled progress on solving SDPs with expressive GNNs~\citep{qian2026gnnsdp}, we propose a sparsity-aware Max-Cut-specific GNN architecture that predicts feasible primal and dual solutions to the SDP relaxation and integrates directly into an exact branch-and-bound framework. Unlike most prior learning-to-optimize approaches, which learn policies to \emph{guide} branch and bound while leaving the bounding step to a conventional solver~\citep{Gasse2019,khalil2022mip,Paulus2022}, our method targets the bounding computation itself. That is, by learning valid bounds rather than relying solely on heuristic decisions, we preserve the performance guarantees of the exact algorithm while substantially reducing overall computing time; see~\cref{fig:overview} for an overview of our proposed framework.

Concretely, our main contributions are as follows.
\begin{enumerate}
    \item We propose a Max-Cut-specific GNN architecture with update step complexity $\mathcal{O}(n^2 + ne)$, where $n$ denotes the number of vertices and $e$ the number of edges, reducing the $\mathcal{O}(n^3)$ bottleneck of prior work \citep{qian2026gnnsdp}.
    
    \item Our architecture predicts both primal-feasible and dual-feasible solutions for the Max-Cut SDP relaxation. Feasibility is enforced by construction, so the predicted dual solution yields a valid upper bound and can be directly integrated into the branch and bound.

    \item The architecture is trained in a self-supervised fashion, directly on Max-Cut instances, without requiring solved SDP relaxations for supervision.

    \item We integrate the learned predictor into an exact branch-and-bound framework for Max-Cut and show speed-ups up to $10.6 \times$ over the same framework using the state-of-the-art interior-point method solver Mosek~\citep{mosek}. To the best of our knowledge, this is the first branch-and-bound implementation that uses a neural network solely to evaluate the problem relaxation.
\end{enumerate}

\emph{Our work establishes that Max-Cut SDP bounds can be learned efficiently, predicted feasibly, and integrated into exact branch and bound to accelerate solving while retaining correctness guarantees, suggesting a broader blueprint for learned surrogates across a large class of structured convex optimization problems. Notably, this is the first result of this type in the literature for the exact solution of a (hard) integer programming problem, setting the stage for a remarkably deeper integration between learning and discrete optimization.}

\subsection{Related work}

\paragraph{GNNs} MPNNs~\citep{Gil+2017,Sca+2009} have been extensively studied. Spatial MPNNs~\citep{bresson2017residual,Duv+2015,Ham+2017,Vel+2018,xu2018how} follow the message-passing paradigm introduced by \citet{Gil+2017}. However, standard MPNNs are fundamentally limited by the expressivity of the \new{$1$-dimensional Weisfeiler--Leman ($\WLk{1}$) test}~\citep{xu2018how, Mor+2019}, motivating various more expressive architectures, most notably \new{higher-order GNNs}~\citep{Mor+2019, Morris2020b, Mar+2019, Mor+2022b}.

\paragraph{GNN for (combinatorial) optimization} In combinatorial optimization, GNNs have demonstrated significant potential. 
For instance, in the context of \new{linear programs} (LPs) and \new{mixed-integer linear programs} (MILPs), GNNs serve a critical role by encoding the underlying problem structure into a constraint-variable bipartite graph~\citep{Gasse2019,ding2020accelerating,chen2022representing,khalil2022mip, qian2024exploring}. Theoretically, a growing body of literature has analyzed the expressivity required for GNNs to approximate the solutions of LPs~\citep{chen2022representing, chen2023mip, qian2024exploring,li2024pdhg,li2024onsmallgnn,li2025towards}, quadratic programs (QPs)~\citep{chen2024qp, wu2024representingqcqp, qian2025principled}, second-order cone programs (SOCPs)~\citep{li2025on}, and, most recently, linear semidefinite programs \citep{qian2026gnnsdp}. Closely related is \citet{OptGNN}, which aligns GNNs with low-rank SDP relaxations to generate rounded heuristic solutions for Max-CSP. In contrast, we explicitly approximate the continuous SDP relaxation to accelerate exact Max-Cut branch-and-bound algorithms.

\paragraph{Neural branching} Branch and bound~\citep{Ach+2005,achterberg2007constraint} is the foundational algorithm for exactly solving MILPs. To accelerate it, early machine learning methods learned surrogate scores~\citep{alvarez2017machine} or ranking functions~\citep{khalil2016learning} for branching. 
More recently, \citet{Gasse2019} trained an MPNN to directly imitate solver-branching behavior, replacing the costly strong-branching score computation. Subsequent works by \citet{zarpellon2021parameterizing,lin2022learning} have extended this by incorporating richer topological information from the search tree. On the theoretical perspective, \citet{chen2024rethinking} and \citet{zhou2025branching} investigated the expressivity required for effective variable selection, concluding that $\WLk{1}$-equivalent architectures are inherently inadequate. 
For a comprehensive survey, please refer to ~\citep{scavuzzo2024machine,chen2024learning}. Despite these advances, existing methods only predict which variable to branch on while still relying on classical LP solvers, and a branch-and-bound algorithm fully based on neural networks has yet to be developed.

\paragraph{Learning for constrained optimization} Another closely related line of work focuses on learning near-optimal solutions directly. Various methods have been proposed to ensure feasible or close-to-feasible predictions for constrained optimization problems. 
A naive approach is to incorporate penalty functions for constraint violations \citep{fioretto2020predicting,pan2022deepopf,qian2024exploring,tang2024learning}. 
More advanced methods include gradient descent correction \citep{Donti2021}, gauge function \citep{Li2023,Tordesillas2023}, homomorphic projection \citep{Liang2023,Liang2024}, decision rule \citep{Gonzalo2026}, differentiable projection layer \citep{Amos2017,Agrawal2019,Chen2021}, and search in feasible region \citep{qian2025towards,rashwan2025enforcing}. 
In parallel with these structural enforcements, several works have investigated primal-dual Lagrangian methods to ensure feasibility \citep{fioretto2020lagrangian,park2023self,tanneau2024dual,kotary2024learning}.  In the context of Max-Cut SDP relaxations, \citet{OptGNN} successfully enforces primal feasibility via a low-rank factorization combined with normalization. Drawing inspiration from the architectural guarantees of \citet{OptGNN} and the dual perspectives of \citet{tanneau2024dual}, our work develops a feasible, self-supervised learning framework tailored specifically for Max-Cut SDPs.

\paragraph{Global solvers for Max-Cut}
Over the years, many SDP-based branch-and-bound methods have been developed to solve Max-Cut problems to global optimality. These methods strengthen the basic SDP relaxation in \cite{goemans1995improved} by incorporating cutting planes (cuts), known as odd-cycle inequalities \citep{Helmberg1998}. Solvers such as BiqMac \citep{biqmac}, BiqBin \citep{biqbin}, BiqCrunch \cite{biqcrunch}, and MADAM \citep{madam} follow this approach. However, the number of triangle inequalities, which are the lowest-order odd-cycle inequalities, scales as $\mathcal{O}(n^3)$, making it impractical to include all of them explicitly. Consequently, these methods rely on first-order algorithms, such as bundle methods, combined with heuristics that selectively add a small subset of cuts. In contrast, our proposed method does not incorporate such cuts and is therefore \emph{not directly comparable} to these solvers in their most advanced configurations. To provide a fair assessment, we benchmark against these solvers both with and without cutting planes enabled in \cref{sec:ab_additional_bab}.

\begin{figure}
    \centering
    \includegraphics[width=\textwidth]{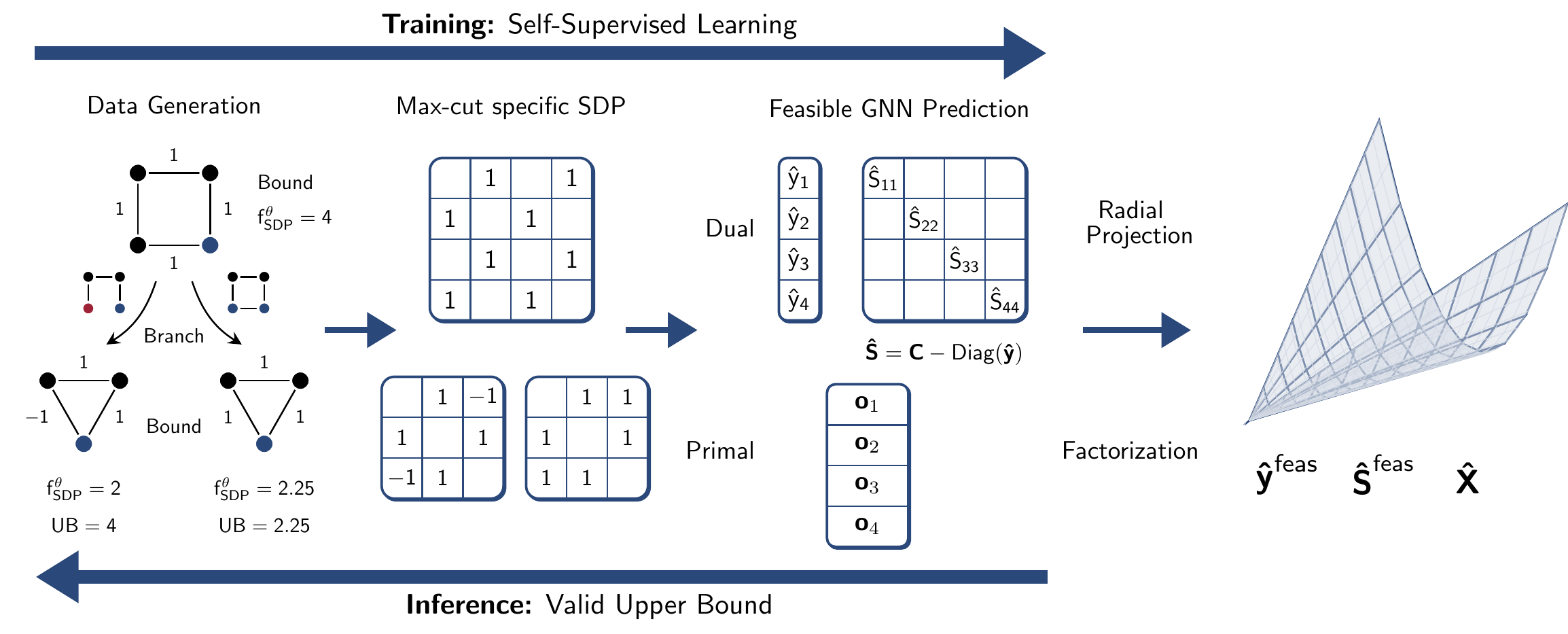} 
    \caption{\textbf{Overview of our neural SDP surrogate for Max-Cut.} \textbf{(1) Data Generation:} Subgraphs are generated dynamically from \emph{random} branch-and-bound search trees. \textbf{(2) GNN Prediction:} $\mcmpnn$ predicts unconstrained primal and dual features. \textbf{(3) Feasibility Guarantee:} A radial projection shifts the $\vec{\hat{y}}$ and $\vec{\hat{S}}$ to the feasible region, while the primal matrix is factorized by construction. \emph{The resulting valid bounds are used to aggressively prune the search tree during exact inference.}}
    \label{fig:overview}
\end{figure}

\subsection{Background}
\label{sec:background}

\paragraph{The Max-Cut problem}
Let $G \coloneqq (V, E)$ be an undirected graph with nodes $V$ ($|V| \coloneqq n$) and edge weights $w_{ij} \in \mathbb{R}$ for $\{i,j\} \in E$. 
The Max-Cut problem seeks to partition the nodes into two disjoint sets to maximize the sum of the weights of the edges crossing the cut. By assigning a binary variable $x_i \in \{-1, 1\}$ to each node to indicate its partition, the problem can be formulated as the quadratic integer program
\begin{equation}
\label{eq:maxcut}
\max_{\vec{x} \in \{-1, 1\}^n} \quad \frac{1}{2} \sum_{\{i,j\} \in E} w_{ij} (1 - x_i x_j) = \frac{1}{4} \vec{x}\trans \vec{L} \vec{x},
\end{equation}
where $\vec{L}$ is the graph Laplacian. Alternatively, using the adjacency matrix as the coefficient matrix $\vec{C}$, this is equivalent to minimizing $\vec{x}\trans \vec{C}\vec{x}$.

\paragraph{Semidefinite relaxation}
Because the integer formulation is \textsf{NP}-complete, a standard practice is to utilize its (SDP) relaxation~\citep{goemans1995improved}. By relaxing the scalar $x_i \in \{-1, 1\}$ to a unit vector $\vec{v}_i \in \mathbb{R}^n$ such that $\lVert \vec{v}_i \rVert = 1$, we define a positive semidefinite (PSD) matrix $\vec{X} \in \mathbb{S}^n_+$ where $X_{ij} = \vec{v}_i\trans \vec{v}_j$. The primal and dual formulations of the Max-Cut SDP relaxation are
\begin{subequations}\label{eq:maxcutsdp}
\begin{minipage}[t]{0.48\textwidth}
\begin{align}
\max_{\vec{X} \in \mathbb{S}^n_+} \quad & \ip{\vec{L}}{\vec{X}} 
\label{eq:maxcutsdpprimal}\\
\text{s.t.} \quad & \mathrm{diag}(\vec{X})=\mathbf{e},
\label{eq:maxcutsdpprimal_diag}
\end{align}
\end{minipage}
\hfill
\begin{minipage}[t]{0.48\textwidth}
\begin{align}
\min_{\vec{y} \in \mathbb{R}^n, \vec{S} \in \mathbb{S}^n_{+}} \quad
& \ip{\vec{e}}{\vec{y}}
\label{eq:maxcutsdpdual}\\
\text{s.t.} \quad & \vec{L} - \mathrm{Diag}(\vec{y}) + \vec{S} = 0.
\label{eq:maxcutsdpdual_eq}
\end{align}
\end{minipage}
\end{subequations}

To analyze this theoretically, we can view Equation~\eqref{eq:maxcutsdpprimal_diag} as a specific instance of the standard linear SDP form, $\min_{\vec{X} \in \mathbb{S}^n_+} \ip{\vec{C}}{\vec{X}} \; \text{ s.t. } \; \Acal(\vec{X}) = \vec{b}$. 
Here, $\vec{C}$ is the objective matrix, and the linear operator $\Acal \colon \mathbb{S}^n \to \mathbb{R}^m$ maps the matrix to $m$ constraints via $\Acal(\vec{X})=\big(\ip{\vec{A}_1}{\vec{X}},\ldots,\ip{\vec{A}_m}{\vec{X}}\big)\trans$. For Max-Cut, the constraints are strictly uniform and diagonal (i.e., $X_{ii} = 1$), meaning $m=n$ and the constraint matrices $\vec{A}_k$ are simply single-entry indicator matrices.

\paragraph{Expressive GNN for Max-Cut SDP}
The $\WLk{1}$ test~\citep{Wei+1968, Cai+1992} iteratively refines a node's color based on the multiset of its neighbors' colors via an injective function, i.e., $\vec{c}^{t}_v \coloneqq \mathsf{hash} ( \vec{c}^{t-1}_v, \dgal{ \vec{c}^{t-1}_u \mid u \in N(v) } )$, where $\dgal{\ldots}$ denotes a multiset. MPNNs are recognized as the neural instantiation of the $\WLk{1}$ test \citep{xu2018how,Mor+2019}. In each layer $t > 0$, an MPNN updates a node's feature $\hb_{v}^t$ by aggregating messages from its local neighborhood $N(v)$ as 
\begin{equation*}\label{def:MPNN_aggregation}
    \vec{m}_v^t = \AGG^{t} \mleft(\dgal[\Big]{\left(\hb_v^{t-1},\hb_{u}^{t-1}, w_{vu}\right) \mid u\in N(v)} \mright), \quad
    \hb_{v}^{t} = \UPD^{t}\mleft(\hb_{v}^{t-1}, \vec{m}_v^t \mright),
\end{equation*}
where $\AGG^{t}$ and $\UPD^{t}$ are parameterized functions, e.g., neural networks. For higher expressivity, the $\WLk{1}$ extends to $\WLk{k}$ and $\FWLk{k}$; see \cref{sec:more_wl} for more detailed background of $\mathsf{WL}$ hierarchy. 

As shown in \citet{qian2026gnnsdp}, representing SDPs necessitates higher-order expressivity equivalent to the $\FWLk{2}$ test; thereby, the $\vctfwl$ framework has been established. It initializes variable embeddings $\vec{v}^0_{ij}$ from the objective $\vec{C}$ and constraint embeddings $\vec{c}^0_k$ from the bounds $\vec{b}$ as
\begin{equation}
\label{def:initv}
\vec{v}^0_{ij} \coloneq \mathsf{init}_{\text{v}}\mleft(C_{ij}, \mathbb{I}_{i=j} \mright), \quad \vec{c}^0_{k} \coloneq \mathsf{init}_{\text{c}}\mleft(b_{k}\mright)
\end{equation}
and applies the following joint update:
\begin{align}
\begin{split}
\label{def:vc2fwl}
\vec{v}_{ij}^t &\coloneq \mathsf{hash} \Bigl( \vec{v}_{ij}^{t-1}, \dgal[\big]{ \dgal{\vec{v}_{uj}^{t-1}, \vec{v}_{iu}^{t-1}} \mid u \in [n]}, \dgal[\big]{(A_{k, ij}, \vec{c}^{t-1}_k) \mid k \in N(ij)} \Bigr) \nonumber \\
\vec{c}^t_{k} &\coloneq \mathsf{hash} \Bigl( \vec{c}^{t-1}_{k}, \dgal[\big]{\left(A_{k,ij}, \vec{v}_{ij}^{t-1} \right) \mid (i,j) \in N(k)} \Bigr).
\end{split}
\end{align}
The update step incurs heavy computational complexity of $\mathcal{O}(n^3 + \text{nnz}(\mathcal{A}))$ per update step for general linear SDP, motivating the derivation of a sparsity-aware variant; see~\cref{sec:gnnmc}.
\paragraph{Branch and Bound} 
To solve the Max-Cut in \eqref{eq:maxcut} to global optimality, the \new{branch-and-bound} (B\&B) method is typically used~\citep{biqmac}. The branch step recursively decomposes a graph into two smaller ones, where one of the vertices is assigned to one partition or the other. As shown in~\cref{fig:overview}, the vertices can be contracted to produce smaller equivalent graphs with updated edge weights $w_{ij}$. Most standard Max-Cut solvers, such as BiqMac \citep{biqmac}, follow this contraction-based approach to recursively reduce the graph size by one. The bound step solves the SDP relaxation of the resulting subproblems to provide an upper bound $f^*_{\mathrm{SDP}}$, while the best feasible cut $\vec{x} \in \{-1, 1\}^n$ gives a lower bound $f_{\mathrm{LB}}$. The Goemans--Williamson algorithm~\citep{goemans1995improved}, which is based on the primal SDP solution, is typically used to obtain good feasible Max-Cut solutions. Hence, the B\&B method progressively closes the gap, and a B\&B node can be safely pruned whenever $f^*_{\mathrm{SDP}} < f_{\mathrm{LB}} + 1$.

\section{A principled GNN architecture to learn feasible bounds}
In this work, we demonstrate that by exploiting the uniform diagonal structure of the Max-Cut relaxation, we can safely bypass this general formulation. We propose a highly optimized, sparsity-aware architecture that preserves the necessary higher expressivity while dramatically reducing the computational bottleneck for branch-and-bound integration.

\subsection{Max-Cut-specific GNN}\label{sec:gnnmc}

Here, we propose a Max-Cut-specific algorithm and its corresponding neural architecture. To that, we first observe that for the Max-Cut problem, the constraints uniformly depend on diagonals. This coincides with the diagonal indicator at the variable node initialization in \cref{def:initv}, meaning the constraints are uniquely and sufficiently marked. Consequently, we can safely drop the corresponding constraint nodes. By removing redundant constraint nodes, this simplification yields our first specific architecture, Max-Cut Weisfeiler--Leman ($\mcwl$). The algorithm initializes variables following \cref{def:initv}, with the color update step defined following, where $\oplus$ denotes any choice of injective and order-invariant aggregation function, e.g., multiset hashing or element-wise summation in neural networks:
\begin{equation*}
\label{eq:mcwl}
\mcwl \colon \quad \vec{v}_{ij}^t \coloneq \mathsf{hash} \Bigl( \vec{v}_{ij}^{t-1}, \dgal[\big]{ \vec{v}_{uj}^{t-1} \oplus \vec{v}_{iu}^{t-1} \mid u \in [n]} \Bigr).
\end{equation*}

In addition, we notice that the objective matrix $\vec{C}$ is often sparse. In fact, based on the proof of the expressivity of the $\vctfwl$, the variable update can be reformulated to be sparsity-aware by incorporating $\vec{C}$ directly into the message passing, resulting in a variant: 
\begin{align}
\dmcwl \colon \quad \vec{v}_{ij}^t \coloneq \mathsf{hash} \Bigl( \vec{v}_{ij}^{t-1}, &\dgal[\big]{ (\vec{v}_{iu}^{t-1}, C_{uj}) \mid u \in [n], C_{uj} \neq 0} \nonumber \\
\oplus & \dgal[\big]{ (\vec{v}_{uj}^{t-1}, C_{iu}) \mid u \in [n], C_{iu} \neq 0} \Bigr). \label{def:dmcwl}
\end{align}

The connection of the multisets via the order-invariant $\oplus$ is to guarantee symmetry $\vec{v}_{ij}^t = \vec{v}_{ji}^t$. 



To determine the running-time complexity of $\dmcwl$, we analyze its aggregation steps. 
The initialization and update is on every index $(i,j) \in [n]^2$, taking $\mathcal{O}(n^2)$ time. 
Furthermore, the sparse attribute aggregations filter exclusively for the non-zero entries of $\vec{C}$. For a given column $j$, let $d_j$ be its number of non-zero entries. Computing the multiset $\dgal[\big]{ (\vec{v}_{iu}^{t-1}, C_{uj}) \mid u \in [n], C_{uj} \neq 0}$ for all $n$ rows requires $\mathcal{O}(n \cdot d_j)$ operations. Summing this cost across all columns yields $\mathcal{O}(n \sum d_j) = \mathcal{O}(ne)$, where $e$ is the total number of non-zero entries.
This concludes the overall running-time complexity per update step, which is $\mathcal{O}(n^2 + ne)$.

The following results show that $\dmcwl$ is expressive enough to represent the optimal solution of a Max-Cut instance. 

\begin{theorem}
\label{thm:sparsewl}
Let $\vec{X}^* \in \mathbb{S}^n_+$ and $\vec{y}^* \in \mathbb{R}^n$ denote the optimal primal and dual solutions for a given Max-Cut SDP relaxation. For any indices $(i, j)$ and $(p, q)$, if the stable colors produced by $\dmcwl$ satisfy $\vec{v}^{\infty}_{ij} = \vec{v}^{\infty}_{pq}$, then the corresponding primal entries satisfy $X^*_{ij} = X^*_{pq}$. Furthermore, if the diagonal stable colors satisfy $\vec{v}^{\infty}_{kk} = \vec{v}^{\infty}_{ll}$, then the dual solution satisfies $y^*_k = y^*_l$.
\end{theorem}

Building on \citet{qian2026gnnsdp}, we can neuralize the algorithms defined above into specific neural network architectures. The variable features are all initialized as $\vec{h}_{ij}^0 \coloneqq \mathsf{INIT}(C_{ij}, \mathbb{I}_{i=j}) \in \mathbb{R}^d$, where $\mathsf{INIT}$ is a parameterized function.

First, translating $\mcwl$ into the $\mcmpnn$ architecture, we encode the pairwise $\oplus$ operation with a $\mathsf{MAP}$ function and the outer multiset aggregation with a neural network $\mathsf{MSG}$, using sum aggregation for both operations:
\begin{equation*}
\label{eq:mcmpnn}
\mcmpnn \quad \vec{h}_{ij}^{t} \coloneq \mathsf{UPD}^{t} \Biggl( \vec{h}_{ij}^{t-1}, \sum_{u \in [n]} \mathsf{MSG}^{t} \Bigl( \mathsf{MAP}^{t}\left(\vec{h}_{uj}^{t-1}\right) + \mathsf{MAP}^{t}\left(\vec{h}_{iu}^{t-1}\right) \Bigr) \Biggr).
\end{equation*}

To make the architecture objective-aware, we neuralize $\dmcwl$ by embedding the matrix entries $C$ directly into the message passing. Merging both multisets after the $\mathsf{MSG}$ function results in the $\dmpnn$ update:
\begin{equation*}
\label{eq:dmpnn}
\dmpnn \quad \vec{h}_{ij}^{t} \coloneq \mathsf{UPD}^{t} \Biggl( \vec{h}_{ij}^{t-1}, \sum_{u \colon C_{iu} \neq 0} \mathsf{MSG}^{t} \bigl( \vec{h}_{uj}^{t-1}, C_{iu} \bigr) + \sum_{u \colon C_{uj} \neq 0} \mathsf{MSG}^{t} \bigl( \vec{h}_{iu}^{t-1}, C_{uj} \bigr) \Biggr).
\end{equation*}


In \cref{sec:proof_neural} we prove that the above architectures can be expressive enough to approximate the optimal solution of a Max-Cut SDP. 

\subsection{Enforce feasibility}
\label{sec:ssl}
Let $\vec{h}_{ij}^T$ denote the final feature tensor produced after $T$ GNN layers. Because the variables operate on the matrix space, we first extract node-level representations $\vec{h}_i$ by aggregating the symmetric pairwise embeddings across the rows (or columns). From these node features, we employ two independent multi-layer perceptron (MLP) heads to formulate the primal and dual predictions.

For the primal variable, an MLP projects $\vec{h}_i$ into a lower-dimensional embedding, which is subsequently $\ell_2$-normalized to yield a unit vector $\vec{o}_i \in \mathbb{R}^r$ such that $\lVert \vec{o}_i \rVert_2 = 1$.
The primal matrix is then constructed via the inner product $\hat{X}_{ij} \coloneq \vec{o}_i\trans \vec{o}_j$, a formulation that structurally guarantees positive semidefiniteness. This prediction is trained by directly optimizing \cref{eq:maxcutsdpprimal}.

For the dual problem, a separate MLP maps $\vec{h}_i$ to an unconstrained scalar prediction $\hat{y}_i \in \mathbb{R}$. However, for the Max-Cut SDP relaxation \eqref{eq:maxcutsdp}, dual feasibility requires the slack matrix to be positive semidefinite: $\vec{S} \in \mathbb{S}^n_{+}$. The raw network outputs $\hat{\vec{y}}$ typically violate this condition. To enforce dual feasibility, we evaluate the predicted slack matrix using the equality constraint \eqref{eq:maxcutsdpdual_eq} as $\hat{\vec{S}} \coloneq \mathrm{Diag}(\hat{\vec{y}}) - \vec{L}$, and apply a radial projection. By uniformly shifting the dual variables by the magnitude of the most negative eigenvalue of $\hat{\vec{S}}$ \footnote{In practice, the spectral decomposition needs to be done only once per instance and is not the computational bottleneck.}, we guarantee a valid, feasible dual solution,
\begin{equation*}
\label{eq:radial_projection}
\hat{\vec{y}}^{\mathrm{feas}} \coloneq \hat{\vec{y}} + \max\{0,\,-\lambda_{\min}(\hat{\vec{S}})\} \mathbf{e},
\end{equation*}
based on which we can directly optimize with \cref{eq:maxcutsdpdual} during training.

Crucially, directly optimizing these objectives aligns seamlessly with our theoretical analysis. To guarantee uniqueness for symmetry analysis, \cref{thm:sparsewl} isolates the minimum Frobenius norm solution, ensuring that nodes indistinguishable by $\dmcwl$ share identical values. In contrast, our self-supervised learning (SSL) framework optimizes the objective directly without this regularization, allowing for a more flexible optimization search space. Nevertheless, because the predicted solution is parameterized by the structurally refined colorings of $\dmcwl$, it inherently preserves the necessary matrix symmetries within this broader space.

\subsection{Branch and bound baseline and GNN-enhanced variants}
\label{sec:BB}
We consider three B\&B solvers for Max-Cut: (i) a vanilla, SDP solver-based B\&B solver; (ii) a hybrid B\&B solver; and (iii) a neural B\&B solver. For completeness, the~\cref{sec:bab_algo} provides additional background on B\&B with the pseudo code of the three solvers shown in \cref{alg:bab}. 

In the vanilla solver (i), each node is evaluated by solving the SDP relaxation with Mosek to obtain a valid upper bound. The resulting primal SDP solution is then rounded by the Goemans--Williamson algorithm to obtain a feasible cut, which can be used to potentially update the incumbent lower bound. Nodes are selected using a best-bound-first strategy, and branching is performed on the most fractional variable. This solver serves as our standard SDP-based baseline. 

The first GNN-enhanced variant is the hybrid solver (ii) that uses the GNN as an early-fathoming oracle before invoking Mosek. Because the proposed GNN predicts a dual-feasible solution by construction, its output always provides a valid upper bound. If this bound already satisfies the pruning condition, the node is fathomed immediately, and the Mosek evaluation is skipped. Otherwise, Mosek is invoked, and the algorithm proceeds exactly as in the vanilla solver.

The second GNN-enhanced variant is the fully neural solver (iii), in which both primal and dual solutions are produced by the GNN. The predicted dual-feasible solution provides a valid upper bound, while the predicted primal-feasible solution is passed to the Goemans--Williamson algorithm to obtain feasible cuts. Thus, the SDP relaxation is never solved by Mosek. Moreover, because GNN inference is lightweight, many child nodes can be collected, batched as graphs, and evaluated in parallel, enabling substantially faster node enumeration while preserving the validity of the bounds.

\section{Experimental study}
\label{sec:results}

The ultimate goal of our architecture is to accelerate exact combinatorial optimization. In this section, we rigorously assess the practical utility of our proposed framework by answering the following key questions:

\begin{enumerate}[label=\textbf{Q\arabic*}, leftmargin=*]
    \item Do the Max-Cut specific GNNs, especially sparsity-aware formulations, yield measurable, practical improvements in computational efficiency?
    \item How does our neural SDP predictor compare against state-of-the-art classical solvers when integrated into a Max-Cut B\&B pipeline?
\end{enumerate}

\subsection{Experiment design}

\paragraph{Comparison of GNN variants on uniform dataset}
To address \textbf{Q1}, we directly compare the predictive performance of $\mcmpnn$ and $\dmpnn$ on synthetic graphs. We generate three distinct datasets of Erd\H{o}s--R\'enyi (ER) graphs \citep{erdds1959random} -- the most common graphs for Max-Cut benchmarking -- varying the number of nodes $N \in \{50, 100\}$ and the edge probabilities $p \in \{0.1, 0.5\}$. Each dataset contains $\num{10000}$ instances and is partitioned into training, validation, and test sets at an $8{:}1{:}1$ ratio. Because our framework utilizes SSL loss, the training phase is entirely unsupervised and requires no ground-truth labels. 
While our training dynamics largely follow \citet{qian2026gnnsdp}, we adapt the framework for SSL by directly optimizing the dual objective.


Finally, to evaluate the predictive accuracy of our models, we compute the exact labels for the test set using the Mosek solver \citep{mosek}. We compare the models using the averaged relative objective gap over the test set:
$\left|\left(\text{obj} - \text{obj}^*\right) / \text{obj}^* \right| \cdot 100\%$, where $\text{obj}$ is the predicted value and $\text{obj}^*$ is the ground-truth.

\paragraph{B\&B experiments on Max-Cut dataset}
For \textbf{Q2}, we evaluate the three variants of B\&B described in Section \ref{sec:BB}: (i) a vanilla SDP-based B\&B solver using Mosek as baseline; (ii) a hybrid B\&B solver that uses GNN for early pruning and invokes Mosek if the node is not pruned; (iii) a neural B\&B solver in which the GNN produces lower bounds and upper bounds. To justify our choice of Mosek as the underlying SDP solver for these evaluations, we carried out ablation experiments comparing multiple SDP solvers in \cref{sec:alationSDPsolver}. For the neural solver, we test both sequential evaluation (no batch) and batched evaluation with batch sizes $K \in \{8, 16, 32\}$. In batched evaluation, the solver instead selects up to $K$ active nodes from the priority queue before performing inference. After branching on these nodes, the resulting $2K$ child-node subgraphs are collected into a mini-batch and evaluated together by the GNN. The predicted bounds are then assigned back to the corresponding child nodes, after which the usual pruning and queue-update steps are applied. Ablation studies of different batch sizes are shown in Sections \ref{sec:ab_heur_free} and \ref{sec:ab_additional_bab}. Experiments were conducted on AMD EPYC 7662 CPU and NVIDIA L40S GPU. 

In every epoch, we generate $\num{20}$ ER seed graphs on the fly from the same distribution of a target instance family. To mimic the subproblems encountered during the B\&B tree, for each seed graph, we sample $\num{20}$ \emph{random branching trajectories}  until 3 free variables and expand them into $20 \times \lvert N-2 \rvert$ subgraphs subject to successive node contractions. 
The B\&B solvers are tested on unseen benchmark instances from Max-Cut libraries reported in \citep{biqmac,biqbin}, including (1) \texttt{g05\_N} for $N \in \{60,80,100\}$ with edge probability $0.5$ and edge weights in $\{0, 1\}$, (2) \texttt{pm1s\_N} for $N \in \{80,100\}$ with edge probability $0.1$ and edge weights in $\{-1, 0, 1\}$; (3) \texttt{w01\_N} for $N \in \{100\}$ with edge probability $0.1$ with integer edge weights in $[-10, 10]$. The model trained within \texttt{g05\_100} distribution is tested on \texttt{g05\_60}, \texttt{g05\_80}, and \texttt{g05\_100}, while the remaining models are trained and tested on their respective distributions. 

\subsection{Results}

\paragraph{Q1: Comparison of GNN variants} \cref{tab:main} demonstrates the objective gaps and the computational overheads during training. Notably, $\dmpnn$ and $\mcmpnn$ exhibit comparable predictive accuracy and training time across all datasets, though $\dmpnn$ requires slightly less VRAM. 

Therefore, we frame our contributions along two parallel tracks: $\dmcwl$ establishes the theory that sub-cubic complexity is possible for Max-Cut SDPs, laying the foundation for future sparse-optimized hardware. Meanwhile, $\mcmpnn$ and $\dmpnn$ both serve as competitive realizations of our theory, emerging as the most practical and efficient choices for deployment.

\begin{table*}[htb!]
\caption{Training computational overhead, measured in seconds per epoch and VRAM usage allocated, alongside predictive performance on the test set, averaged over 3 random seeds. The best results are highlighted in bold.}
\label{tab:main}
\centering
\resizebox{0.8\textwidth}{!}{
\begin{tabular}{lcccc}
\toprule
\multirow{2}{*}{Target} & \multirow{2}{*}{Model} & \multicolumn{3}{c}{\textbf{Problems}} \\
& & N=50, p=0.1 & N=50, p=0.5 & N=100, p=0.1 \\
\midrule

\multirow{2}{*}{Obj gap (\%)} & $\mcmpnn$ 
& \textbf{0.126$\pm$\scriptsize0.005}    %
& 0.101$\pm$\scriptsize0.001    %
& \textbf{0.598$\pm$\scriptsize0.005}    \\ %
& $\dmpnn$  
& 0.128$\pm$\scriptsize0.005    %
& \textbf{0.099$\pm$\scriptsize0.001}    %
& 0.628$\pm$\scriptsize0.004    \\ 
\midrule

\multirow{2}{*}{Time sec/epoch} & $\mcmpnn$  
& 6.5    %
& \textbf{6.3}    %
& \textbf{27.9}    \\ %
& $\dmpnn$  
& \textbf{6.0}    %
& \textbf{6.3}    %
& 28.2    \\ %
\midrule

\multirow{2}{*}{VRAM (GB)} & $\mcmpnn$  
& 23.6    %
& 23.6    %
& 82.0   \\ %
& $\dmpnn$  
& \textbf{21.6}    %
& \textbf{21.6}    %
& \textbf{81.6}    \\
\bottomrule
\end{tabular}
}
\end{table*}

Additional ablation studies, including sensitivity analyses of network depth, out-of-distribution generalization, and normalization strategies ($\mathsf{LayerNorm}$, $\mathsf{GraphNorm}$ etc.), are provided in \cref{sec:add_exps}.

\paragraph{Q2: Branch-and-bound experiments}
We first conduct a controlled experiment to isolate the effect of the GNN-predicted upper bounds. In this experiment, we disable the Goemans--Williamson rounding step in all three B\&B solvers (i)--(iii) and initialize the global lower bound with the true optimal Max-Cut value at the root node. This setup removes the influence of primal heuristics and incumbent updates, so differences in performance can be attributed directly to the quality of the upper bounds used for fathoming.   The complete heuristic-free B\&B statistics for all three solvers are provided in Section~\ref{sec:ab_heur_free}. In \cref{tab:babhybrid}, we highlight that the hybrid solver (ii) predicts sufficiently tight upper bounds to identify more than $83\%$ of SDP-prunable nodes during B\&B.

\begin{table*}[htb!]
\caption{Average fraction of SDP-prunable nodes pruned by GNN in the B\&B without Goemans--Williamson algorithm using the hybrid solver on established instances. The best-performing configuration, $\mcmpnn$ + $\mathsf{LayerNorm}$, is used for the hybrid solver. }
\centering
\label{tab:babhybrid}
\resizebox{0.6\textwidth}{!}{
\begin{tabular}{lcc}
\toprule
Instance Family & Method & \% Nodes Pruned by GNN \\
\midrule
\multirow{1}{*}{\texttt{g05\_60}} 
& Hybrid
& 85.5\%    \\ %
\midrule

\multirow{1}{*}{\texttt{g05\_80}} 
& Hybrid
& 86.8\%    \\ %
\midrule

\multirow{1}{*}{\texttt{g05\_100}} 
& Hybrid
& 83.8\%    \\ %
\bottomrule
\end{tabular}
}
\end{table*}

Having isolated the effect of upper-bound quality, we next evaluate whether the learned bounds remain useful in a full B\&B pipeline where incumbent solutions are generated dynamically. For this experiment, we focus on variants (i) and (iii): the vanilla solver and the fully neural solver. In both cases, feasible Max-Cut solutions are obtained via the Goemans--Williamson algorithm, using primal SDP solutions from Mosek for the vanilla solver and from the GNN for the neural solver. The results are reported in \cref{tab:bab_main}, where the neural solver uses the best-performing configuration selected in the ablation studies in \cref{sec:ab_additional_bab}: $\mcmpnn$ with \textsf{LayerNorm} and batched inference of size 32.

Overall, the neural solver evaluates slightly more nodes than the vanilla SDP-based solver, reflecting the fact that its learned bounds are somewhat less tight than the exact SDP solver. 
However, this modest increase in tree size is more than offset by the much lower cost of each node evaluation. This per-node cost reduction is driven by two complementary effects. First, the GNN replaces repeated SDP solves with lightweight inference. Second, batching enables multiple GNNs to be evaluated together in a single mini-batch, which improves hardware utilization. 
The speed-ups are consistent across all tested instance families and become more pronounced on larger instances, where SDP solves are increasingly expensive. The additional ablation studies of the two complementary effects are provided in~\cref{sec:ab_heur_free,sec:ab_additional_bab}.

The strong performance of the model trained on \texttt{g05\_100}-in-distribution graphs when tested on \texttt{g05\_60} and \texttt{g05\_80} suggests that the learned SDP proxy effectively transfers across graph sizes within the same instance family. This generalization is nontrivial because, although the training data contain contracted subgraphs of comparable sizes, they are typically denser and no longer have purely binary edge weights. A more detailed analysis of generalization is provided in \cref{sec:ab_norm}.

\begin{table*}[htb!]
\caption{Branch-and-bound results on established instances. The \#node and time are averages over each instance family. The best-performing configuration, $\mcmpnn$ + $\mathsf{LayerNorm}$ with batched inference of size $32$, is reported for the neural solver.}
\centering
\label{tab:bab_main}
\resizebox{0.9\textwidth}{!}{
\begin{tabular}{lcccc}
\toprule
Instance Family & Method & \# Nodes Evaluated & Solving Time (s) & Speed-up ($\times$) \\
\midrule
\multirow{2}{*}{\texttt{g05\_60}} 
& Vanilla 
& 3566   %
& 53.1   & 1 \\ %
& Neural
& \textbf{4562}   %
& \textbf{9.5} & \textbf{5.6}   \\ %
\midrule

\multirow{2}{*}{\texttt{g05\_80}} 
& Vanilla 
& 49128    %
& 1170.3 & 1    \\ %
& Neural
&  \textbf{53975}  %
&  \textbf{155.7} & \textbf{7.5}  \\ %
\midrule

\multirow{2}{*}{\texttt{g05\_100}} 
& Vanilla 
& 1087290    %
& 40930.1 & 1    \\ %
& Neural
&  \textbf{1272801}  %
&   \textbf{4509.3} & \textbf{9.1}  \\ %
\midrule

\multirow{2}{*}{\texttt{pm1s\_80}} 
& Vanilla 
&  25656   %
&   609.5 & 1  \\ %
& Neural
&   \textbf{27809} %
&    \textbf{80.8} & \textbf{7.5} \\ %
\midrule

\multirow{2}{*}{\texttt{pm1s\_100}} 
& Vanilla 
& 254408    %
& 10382.5 & 1    \\ %
& Neural
&  \textbf{302314}  %
&  \textbf{1058.2} & \textbf{9.8}  \\ %
\midrule

\multirow{2}{*}{\texttt{w01\_100}} 
& Vanilla 
& 269407    %
& 11348.6 & 1    \\ %
& Neural
& \textbf{305720}   %
&  \textbf{1073.9} & \textbf{10.6}  \\ %
\bottomrule
\end{tabular}
}
\end{table*}

\section{Conclusion}

We introduced a feasibility-preserving GNN proxy for the semidefinite relaxations arising in exact Max-Cut branch and bound. By exploiting the sparse structure of the Max-Cut SDP, our architecture avoids the cubic bottleneck of more general GNNs for SDP while retaining the expressivity needed to represent optimal primal and dual solutions. The resulting model predicts primal-feasible and dual-feasible SDP solutions by construction, enabling valid upper bounds and self-supervised training without requiring solved SDP instances as labels. Empirically, our results show that learned SDP bounds can be sufficiently accurate to substantially reduce the cost of bounding in branch and bound. In particular, the proposed neural solver accelerates SDP-based exact Max-Cut solving by up to $10.6\times$ compared with the same framework using Mosek, while preserving correctness via dual feasibility. These findings suggest that learned solvers need not be limited to heuristic guidance, such as branching or cut selection, but can instead replace core relaxation computations when feasibility and validity are enforced through a tailored GNN architecture. 

\paragraph{Limitations and possible future directions} 
The proposed framework exhibits two primary limitations. First, its reliance on generating and training on in-distribution data restricts the model's ability to generalize to out-of-distribution graphs. To address this, we plan to explore foundation models pretrained across diverse graph classes. Second, although our neural solver effectively accelerates exact combinatorial optimization, it does not yet consistently surpass state-of-the-art classical exact solvers such as BiqCrunch \cite{biqcrunch}, which have benefited from decades of research on Max-Cut and derive much of their strength from advanced branch-and-cut paradigms. To bridge this performance gap, we aim to investigate tighter integration between neural SDP solvers and classical cutting-plane methods.


\begin{ack}
C.L. would like to acknowledge the financial support from the Office of Naval Research (ONR) award N000142412641. A.L. is grateful for the support of ONR, award N00014-24-1-2645. C.M. and C.Q. are partially funded by a DFG Emmy Noether grant (468502433) and RWTH Junior Principal Investigator Fellowship under Germany’s Excellence Strategy
\end{ack}

\clearpage


\bibliographystyle{abbrvnat}

\bibliography{reference}

@article{Helmberg1998,
  title={Solving quadratic (0, 1)-problems by semidefinite programs and cutting planes},
  author={Helmberg, Christoph and Rendl, Franz},
  journal={Mathematical programming},
  volume={82},
  number={3},
  pages={291--315},
  year={1998},
  publisher={Springer}
}

@article{madam,
   author = {Timotej Hrga and Janez Povh},
   doi = {10.1007/s10589-021-00310-6},
   issn = {1573-2894},
   issue = {2},
   journal = {Computational Optimization and Applications},
   pages = {347-375},
   title = {MADAM: a parallel exact solver for max-cut based on semidefinite programming and ADMM},
   volume = {80},
   year = {2021}
}

@article{biqcrunch,
author = {Krislock, Nathan and Malick, J\'{e}r\^{o}me and Roupin, Fr\'{e}d\'{e}ric},
title = {BiqCrunch: A Semidefinite Branch-and-Bound Method for Solving Binary Quadratic Problems},
year = {2017},
issue_date = {December 2017},
publisher = {Association for Computing Machinery},
address = {New York, NY, USA},
volume = {43},
number = {4},
issn = {0098-3500},
doi = {10.1145/3005345},
journal = {ACM Trans. Math. Softw.},
month = jan,
articleno = {32},
numpages = {23}
}

@article{chen2024learning,
  title={Learning to optimize: A tutorial for continuous and mixed-integer optimization},
  author={Chen, Xiaohan and Liu, Jialin and Yin, Wotao},
  journal={Science China Mathematics},
  volume={67},
  number={6},
  pages={1191--1262},
  year={2024},
  publisher={Springer}
}

@article{goemans1995improved,
  title={Improved approximation algorithms for maximum cut and satisfiability problems using semidefinite programming},
  author={Goemans, Michel X and Williamson, David P},
  journal={Journal of the ACM (JACM)},
  volume={42},
  number={6},
  pages={1115--1145},
  year={1995},
  publisher={ACM New York, NY, USA}
}

@inproceedings{rendl2007branch,
  title={A branch and bound algorithm for max-cut based on combining semidefinite and polyhedral relaxations},
  author={Rendl, Franz and Rinaldi, Giovanni and Wiegele, Angelika},
  booktitle={International Conference on Integer Programming and Combinatorial Optimization},
  pages={295--309},
  year={2007},
  organization={Springer}
}

@article{biqbin,
  title={BiqBin: a parallel branch-and-bound solver for binary quadratic problems with linear constraints},
  author={Gusmeroli, Nicolo and Hrga, Timotej and Lu{\v{z}}ar, Borut and Povh, Janez and Siebenhofer, Melanie and Wiegele, Angelika},
  journal={ACM Transactions on Mathematical Software (TOMS)},
  volume={48},
  number={2},
  pages={1--31},
  year={2022},
  publisher={ACM New York, NY}
}

@article {biqmac,
author={Rendl, Franz and Rinaldi, Giovanni and Wiegele, Angelika},
title={Solving {M}ax-{C}ut to Optimality by Intersecting Semidefinite and Polyhedral Relaxations},
journal = {Math. Programming},
volume = {121},
year = {2010},
number = {2},
pages = {307}
}

@inproceedings{Chen2021,
  title={Enforcing policy feasibility constraints through differentiable projection for energy optimization},
  author={Chen, Bingqing and Donti, Priya L and Baker, Kyri and Kolter, J Zico and Berg{\'e}s, Mario},
  booktitle={Proceedings of the Twelfth ACM International Conference on Future Energy Systems},
  pages={199--210},
  year={2021}
}

@article{Agrawal2019,
  title={Differentiable convex optimization layers},
  author={Agrawal, Akshay and Amos, Brandon and Barratt, Shane and Boyd, Stephen and Diamond, Steven and Kolter, J Zico},
  journal={Advances in Neural Information Processing Systems},
  year={2019}
}

@inproceedings{Amos2017,
  title={Optnet: Differentiable optimization as a layer in neural networks},
  author={Amos, Brandon and Kolter, J Zico},
  booktitle={International Conference on Machine Learning},
  year={2017},
  organization={PMLR}
}

@article{Li2023,
   author = {Meiyi Li and Soheil Kolouri and Javad Mohammadi},
   doi = {10.1109/access.2023.3285199},
   issn = {21693536},
   journal = {IEEE Access},
   pages = {59995-60004},
   publisher = {Institute of Electrical and Electronics Engineers Inc.},
   title = {Learning to solve optimization problems with hard linear constraints},
   volume = {11},
   year = {2023}
}

@inproceedings{Gonzalo2026,
  title={Enforcing Hard Linear Constraints in Deep Learning Models with Decision Rules},
  author={Flores, Gonzalo E Constante and Chen, Hao and Li, Can},
  booktitle={The Thirty-ninth Annual Conference on Neural Information Processing Systems},
  year={2025}
}

@InProceedings{Liang2023,
  title = 	 {Low Complexity Homeomorphic Projection to Ensure Neural-Network Solution Feasibility for Optimization over ({N}on-){C}onvex Set},
  author =       {Liang, Enming and Chen, Minghua and Low, Steven},
  booktitle = 	 {International Conference on Machine Learning},
  year = 	 {2023},
  publisher =    {PMLR},
}

@article{Liang2024,
   author = {Enming Liang and Minghua Chen and Steven H Low},
   journal = {Journal of Machine Learning Research},
   pages = {1-55},
   title = {Homeomorphic Projection to Ensure Neural-Network Solution Feasibility for Constrained Optimization},
   volume = {25},
   year = {2024}
}

@article{Tordesillas2023,
      title={RAYEN: Imposition of Hard Convex Constraints on Neural Networks}, 
      author={Jesus Tordesillas and Jonathan P. How and Marco Hutter},
      journal={arXiv preprint arXiv:2307.08336},
      year={2023}
}

@inproceedings{Donti2021,
  title={{DC3}: A learning method for optimization with hard constraints},
  author={Donti, Priya and Rolnick, David and Kolter, J Zico},
  booktitle={International Conference on Learning Representations},
  year={2021}
}

@article{bengio2021machine,
  title={Machine learning for combinatorial optimization: a methodological tour d'horizon},
  author={Bengio, Yoshua and Lodi, Andrea and Prouvost, Antoine},
  journal={European Journal of Operational Research},
  volume={290},
  number={2},
  pages={405--421},
  year={2021},
  publisher={Elsevier}
}

@article{burer2003nonlinear,
  title={A nonlinear programming algorithm for solving semidefinite programs via low-rank factorization},
  author={Burer, Samuel and Monteiro, Renato DC},
  journal={Mathematical programming},
  volume={95},
  number={2},
  pages={329--357},
  year={2003},
  publisher={Springer}
}

@article{Ach+2005,
  author    = {T. Achterberg and
               T. Koch and
               A. Martin},
  title     = {Branching rules revisited},
  journal   = {Operations Research Letters},
  volume    = {33},
  number    = {1},
  pages     = {42--54},
  year      = {2005}
}

@phdthesis{achterberg2007constraint,
  title={Constraint integer programming},
  author={Achterberg, Tobias},
  year={2007}
}

@article{Duv+2015,
  title={Convolutional Networks on Graphs for Learning Molecular Fingerprints},
  author={Duvenaud, David and Maclaurin, Dougal and Aguilera-Iparraguirre, Jorge and G{\'o}mez-Bombarelli, Rafael and Hirzel, Timothy and Aspuru-Guzik, Al{\'a}n and Adams, Ryan P},
  journal={arXiv preprint arXiv:1509.09292},
  year={2015}
}

@book{Kor+2012,
	author = {Korte, B. and Vygen, J.},
	title = {Combinatorial Optimization: Theory and Algorithms},
	year = {2012},
	edition = {5th},
	publisher = {Springer}
}

@inproceedings{xu2018how,
title={How Powerful are Graph Neural Networks?},
author={Keyulu Xu and Weihua Hu and Jure Leskovec and Stefanie Jegelka},
booktitle={International Conference on Learning Representations},
year={2019}
}

@InProceedings{Mor+2022b,
  author    = {Christopher Morris and Gaurav Rattan and Sandra Kiefer and Siamak Ravanbakhsh},
  booktitle = {ICML},
  title     = {{SpeqNets}: Sparsity-aware permutation-equivariant graph networks},
  year      = {2022}
}

@Article{Cap+2023,
  author  = {Quentin Cappart and Didier Ch{\'{e}}telat and Elias B. Khalil and Andrea Lodi and Christopher Morris and Petar Veli\v{c}kovi\'{c}},
  journal = {Journal of Machine Learning Research},
  title   = {Combinatorial Optimization and Reasoning with Graph Neural Networks},
  year    = {2023},
  number  = {130},
  pages   = {1--61},
  volume  = {24},
}

@article{Nai+2020,
  author       = {Vinod Nair and
                  Sergey Bartunov and
                  Felix Gimeno and
                  Ingrid von Glehn and
                  Pawel Lichocki and
                  Ivan Lobov and
                  Brendan O'Donoghue and
                  Nicolas Sonnerat and
                  Christian Tjandraatmadja and
                  Pengming Wang and
                  Ravichandra Addanki and
                  Tharindi Hapuarachchi and
                  Thomas Keck and
                  James Keeling and
                  Pushmeet Kohli and
                  Ira Ktena and
                  Yujia Li and
                  Oriol Vinyals and
                  Yori Zwols},
  title        = {Solving Mixed Integer Programs Using Neural Networks},
  journal      = {CoRR},
  volume       = {abs/2012.13349},
  year         = {2020},
  eprinttype   = {arXiv},
  eprint       = {2012.13349},
  bibsource    = {dblp computer science bibliography, https://dblp.org}
}

@InProceedings{Morris2020b,
  author    = {C. Morris and G. Rattan and P. Mutzel},
  booktitle = {NeurIPS},
  title     = {{Weisfeiler and Leman} Go Sparse: Towards Higher-Order Graph Embeddings},
  year      = {2020},
}

@InProceedings{Mor+2019,
  author    = {Morris, Christopher and Ritzert, Martin and Fey, Matthias and Hamilton, William L and Lenssen, Jan Eric and Rattan, Gaurav and Grohe, Martin},
  booktitle = {AAAI},
  title     = {{W}eisfeiler and {L}eman go neural: Higher-order graph neural networks},
  year      = {2019}
}

@InProceedings{Mar+2019,
  author    = {Haggai Maron and Heli Ben{-}Hamu and Hadar Serviansky and Yaron Lipman},
  booktitle = {NeurIPS},
  title     = {Provably Powerful Graph Networks},
  year      = {2019}
}

@inproceedings{Vel+2018,
title={Graph Attention Networks},
author={Petar Veličković and Guillem Cucurull and Arantxa Casanova and Adriana Romero and Pietro Liò and Yoshua Bengio},
booktitle={International Conference on Learning Representations},
year={2018}
}

@article{Ham+2017,
  title={Inductive representation learning on large graphs},
  author={Hamilton, Will and Ying, Zhitao and Leskovec, Jure},
  journal={NeurIPS},
  year={2017}
}

@Book{Gro2017,
  author    = {Grohe, Martin},
  publisher = {Cambridge University Press},
  title     = {Descriptive complexity, canonisation, and definable graph structure theory},
  year      = {2017}
}

@InProceedings{Kin+2015,
  author    = {Diederik P. Kingma and Jimmy Ba},
  booktitle = {ICLR},
  title     = {Adam: {A} Method for Stochastic Optimization},
  year      = {2015},
}

@Article{Sca+2009,
  author    = {Scarselli, Franco and Gori, Marco and Tsoi, Ah Chung and Hagenbuchner, Markus and Monfardini, Gabriele},
  journal   = {IEEE Transactions on Neural Networks},
  title     = {The graph neural network model},
  year      = {2008},
  number    = {1},
  pages     = {61--80},
  volume    = {20},
  publisher = {IEEE},
}

@Article{Cai+1992,
  author    = {Cai, Jin-Yi and F{\"u}rer, Martin and Immerman, Neil},
  journal   = {Combinatorica},
  title     = {An optimal lower bound on the number of variables for graph identification},
  year      = {1992},
  number    = {4},
  pages     = {389--410},
  volume    = {12},
  publisher = {Springer},
}

@InProceedings{Imm+1990,
  author    = {Immerman, N. and Lander, E.},
  booktitle = {Complexity Theory Retrospective: {I}n Honor of Juris Hartmanis on the Occasion of His Sixtieth Birthday, July 5, 1988},
  title     = {Describing Graphs: {A} First-Order Approach to Graph Canonization},
  year      = {1990},
  pages     = {59--81},
}

@Unpublished{Bab1979,
  author = {L{\'a}szl{\'o} Babai},
  note   = {University of Toronto, Department of Computer Science. Mimeographed lecture notes, October 1979},
  title  = {Lectures on Graph Isomorphism},
  year   = {1979},
}

@Article{Wei+1968,
  author  = {Weisfeiler, Boris and Leman, Andrei},
  journal = {nti, Series},
  title   = {The reduction of a graph to canonical form and the algebra which appears therein},
  year    = {1968},
  number  = {9},
  pages   = {12--16},
  volume  = {2},
}

@inproceedings{Gil+2017,
  title={Neural message passing for quantum chemistry},
  author={Gilmer, Justin and Schoenholz, Samuel S and Riley, Patrick F and Vinyals, Oriol and Dahl, George E},
  booktitle={International Conference on Machine Learning},
  year={2017}
}

@article{bresson2017residual,
  title={Residual gated graph {ConvNets}},
  author={Bresson, Xavier and Laurent, Thomas},
  journal={arXiv preprint arXiv:1711.07553v2},
  year={2017}
}

@inproceedings{chen2022representing,
title={On Representing Linear Programs by Graph Neural Networks},
author={Ziang Chen and Jialin Liu and Xinshang Wang and Wotao Yin},
booktitle={The Eleventh International Conference on Learning Representations },
year={2023}
}

@article{chen2024rethinking,
  title={Rethinking the capacity of graph neural networks for branching strategy},
  author={Chen, Ziang and Liu, Jialin and Chen, Xiaohan and Wang, Xinshang and Yin, Wotao},
  journal={Advances in Neural Information Processing Systems},
  year={2024}
}

@inproceedings{chen2023mip,
title={On Representing Mixed-Integer Linear Programs by Graph Neural Networks},
author={Ziang Chen and Jialin Liu and Xinshang Wang and Wotao Yin},
booktitle={The Eleventh International Conference on Learning Representations },
year={2023}
}

@inproceedings{chen2024qp,
title={Expressive Power of Graph Neural Networks for (Mixed-Integer) Quadratic Programs},
author={Ziang Chen and Xiaohan Chen and Jialin Liu and Xinshang Wang and Wotao Yin},
booktitle={Forty-second International Conference on Machine Learning},
year={2025}
}

@article{wu2024representingqcqp,
title={On Representing Convex Quadratically Constrained Quadratic Programs via Graph Neural Networks},
author={Chenyang Wu and Qian Chen and Akang Wang and Tian Ding and Ruoyu Sun and Wenguo Yang and Qingjiang Shi},
journal={Transactions on Machine Learning Research},
issn={2835-8856},
year={2025}
}

@inproceedings{qian2024exploring,
  title={Exploring the power of graph neural networks in solving linear optimization problems},
  author={Qian, Chendi and Ch{\'e}telat, Didier and Morris, Christopher},
  booktitle={AISTATS},
  year={2024},
}

@inproceedings{li2024onsmallgnn,
 author = {Li, Qian and Ding, Tian and Yang, Linxin and Ouyang, Minghui and Shi, Qingjiang and Sun, Ruoyu},
 booktitle = {NeurIPS},
 title = {On the Power of Small-size Graph Neural Networks for Linear Programming},
 year = {2024}
}

@inproceedings{li2024pdhg,
author = {Li, Bingheng and Yang, Linxin and Chen, Yupeng and Wang, Senmiao and Mao, Haitao and Chen, Qian and Ma, Yao and Wang, Akang and Ding, Tian and Tang, Jiliang and Sun, Ruoyu},
title = {PDHG-unrolled learning-to-optimize method for large-scale linear programming},
year = {2024},
publisher = {JMLR.org},
booktitle = {Proceedings of the 41st International Conference on Machine Learning},
articleno = {1173},
numpages = {17},
location = {Vienna, Austria},
series = {ICML'24}
}

@inproceedings{ding2020accelerating,
  title={Accelerating primal solution findings for mixed integer programs based on solution prediction},
  author={Ding, Jian-Ya and Zhang, Chao and Shen, Lei and Li, Shengyin and Wang, Bing and Xu, Yinghui and Song, Le},
  booktitle={AAAI},
  year={2020}
}

@inproceedings{khalil2022mip,
  title={Mip-gnn: A data-driven framework for guiding combinatorial solvers},
  author={Khalil, Elias B and Morris, Christopher and Lodi, Andrea},
  booktitle={AAAI},
  year={2022}
}

@inproceedings{khalil2016learning,
  title={Learning to branch in mixed integer programming},
  author={Khalil, Elias and Le Bodic, Pierre and Song, Le and Nemhauser, George and Dilkina, Bistra},
  booktitle={Proceedings of the AAAI conference on artificial intelligence},
  volume={30},
  number={1},
  year={2016}
}

@article{zhou2025branching,
  title={Branching Strategies Based on Subgraph GNNs: A Study on Theoretical Promise versus Practical Reality},
  author={Zhou, Junru and Wang, Yicheng and Li, Pan},
  journal={arXiv preprint arXiv:2512.09355},
  year={2025}
}

@inproceedings{
li2025towards,
title={Towards Explaining the Power of Constant-depth Graph Neural Networks for Structured Linear Programming},
author={Qian Li and Minghui Ouyang and Tian Ding and Yuyi Wang and Qingjiang Shi and Ruoyu Sun},
booktitle={ICLR},
year={2025}
}

@article{OptGNN,
  title={Are graph neural networks optimal approximation algorithms?},
  author={Yau, Morris and Karalias, Nikolaos and Lu, Eric and Xu, Jessica and Jegelka, Stefanie},
  journal={NeurIPS},
  year={2024}
}

@inproceedings{Li+2023,
  author       = {Sirui Li and
                  Wenbin Ouyang and
                  Max B. Paulus and
                  Cathy Wu},
  title        = {Learning to Configure Separators in Branch-and-Cut},
  booktitle    = {NeurIPS},
  year         = {2023}
}

@inproceedings{Dez+2023,
  author       = {Arnaud Deza and
                  Elias B. Khalil},
  title        = {Machine Learning for Cutting Planes in Integer Programming: {A} Survey},
  booktitle    = {IJCAI},
  year         = {2023}
}

@article{wen2010alternating,
  title={Alternating direction augmented Lagrangian methods for semidefinite programming},
  author={Wen, Zaiwen and Goldfarb, Donald and Yin, Wotao},
  journal={Mathematical Programming Computation},
  volume={2},
  number={3},
  pages={203--230},
  year={2010},
  publisher={Springer}
}

@article{o2021operator,
  title={Operator splitting for a homogeneous embedding of the linear complementarity problem},
  author={O'Donoghue, Brendan},
  journal={SIAM Journal on Optimization},
  volume={31},
  number={3},
  pages={1999--2023},
  year={2021},
  publisher={SIAM}
}

@article{mosek,
  title={Mosek optimizer API for python},
  author={ApS, Mosek},
  journal={Version},
  volume={9},
  number={17},
  pages={6--4},
  year={2022}
}

@article{vandenberghe1996semidefinite,
  title={Semidefinite programming},
  author={Vandenberghe, Lieven and Boyd, Stephen},
  journal={SIAM review},
  volume={38},
  number={1},
  pages={49--95},
  year={1996},
  publisher={SIAM}
}

@book{wolkowicz2012handbook,
  title={Handbook of semidefinite programming: theory, algorithms, and applications},
  author={Wolkowicz, Henry and Saigal, Romesh and Vandenberghe, Lieven},
  volume={27},
  year={2012},
  publisher={Springer Science \& Business Media}
}

@article{erdds1959random,
  title={On random graphs I},
  author={Erd{\H{o}}s, Paul and R{\'e}nyi, Alfr{\'e}d},
  journal={Publ. math. debrecen},
  volume={6},
  number={290-297},
  pages={18},
  year={1959}
}

@article{barahona1988application,
  author    = {Francisco Barahona and Martin Gr{\"o}tschel and Michael J{\"u}nger and Gerhard Reinelt},
  title     = {An Application of Combinatorial Optimization to Statistical Physics and Circuit Layout Design},
  journal   = {Operations Research},
  volume    = {36},
  number    = {3},
  pages     = {493--513},
  year      = {1988},
  doi       = {10.1287/opre.36.3.493}
}

@incollection{poljak1995maxcutsurvey,
  author    = {Svatopluk Poljak and Zsolt Tuza},
  title     = {The Max-Cut Problem --- A Survey},
  booktitle = {Special Year on Combinatorial Optimization},
  publisher = {American Mathematical Society},
  year      = {1995}
}

@article{wang2024tuningfree,
  title={A tuning-free primal-dual splitting algorithm for large-scale semidefinite programming},
  author={Wang, Yinjun and Lan, Haixiang and Ye, Yinyu},
  journal={arXiv preprint arXiv:2402.00311},
  year={2024}
}

@article{chambolle2016ergodic,
  title={On the ergodic convergence rates of a first-order primal--dual algorithm},
  author={Chambolle, Antonin and Pock, Thomas},
  journal={Mathematical Programming},
  volume={159},
  number={1},
  pages={253--287},
  year={2016},
  publisher={Springer}
}

@article{han2025low,
  title={A low-rank admm splitting approach for semidefinite programming},
  author={Han, Qiushi and Li, Chenxi and Lin, Zhenwei and Chen, Caihua and Deng, Qi and Ge, Dongdong and Liu, Huikang and Ye, Yinyu},
  journal={INFORMS Journal on Computing},
  year={2025},
  publisher={INFORMS}
}

@book{horn1994topics,
  title={Topics in matrix analysis},
  author={Horn, Roger A and Johnson, Charles R},
  year={1994},
  publisher={Cambridge university press}
}

@article{ba2016layer,
  title={Layer normalization},
  author={Ba, Jimmy Lei and Kiros, Jamie Ryan and Hinton, Geoffrey E},
  journal={arXiv preprint arXiv:1607.06450},
  year={2016}
}

@inproceedings{cai2021graphnorm,
  title={Graphnorm: A principled approach to accelerating graph neural network training},
  author={Cai, Tianle and Luo, Shengjie and Xu, Keyulu and He, Di and Liu, Tie-yan and Wang, Liwei},
  booktitle={International Conference on Machine Learning},
  year={2021}
}

@inproceedings{park2023self,
  title={Self-supervised primal-dual learning for constrained optimization},
  author={Park, Seonho and Van Hentenryck, Pascal},
  booktitle={AAAI},
  year={2023}
}

@article{tanneau2024dual,
  title={Dual lagrangian learning for conic optimization},
  author={Tanneau, Mathieu and Van Hentenryck, Pascal},
  journal={NeurIPS},
  year={2024}
}

@inproceedings{
li2025on,
title={On the Expressivity of {GNN} for Solving Second Order Cone Programming},
author={Ruizhe Li and Enming Liang and Minghua Chen},
booktitle={NeurIPS Workshop on GPU-Accelerated and Scalable Optimization},
year={2025}
}

@InProceedings{Paulus2022,
  title = 	 {Learning to Cut by Looking Ahead: Cutting Plane Selection via Imitation Learning},
  author =       {Paulus, Max B and Zarpellon, Giulia and Krause, Andreas and Charlin, Laurent and Maddison, Chris},
  booktitle = 	 {Proceedings of the 39th International Conference on Machine Learning},
  pages = 	 {17584--17600},
  year = 	 {2022},
  editor = 	 {Chaudhuri, Kamalika and Jegelka, Stefanie and Song, Le and Szepesvari, Csaba and Niu, Gang and Sabato, Sivan},
  volume = 	 {162},
  series = 	 {Proceedings of Machine Learning Research},
  month = 	 {17--23 Jul},
  publisher =    {PMLR}
}

@inbook{Gasse2019,
author = {Gasse, Maxime and Ch\'{e}telat, Didier and Ferroni, Nicola and Charlin, Laurent and Lodi, Andrea},
title = {Exact combinatorial optimization with graph convolutional neural networks},
year = {2019},
publisher = {Curran Associates Inc.},
address = {Red Hook, NY, USA},
booktitle = {Proceedings of the 33rd International Conference on Neural Information Processing Systems},
articleno = {1396},
numpages = {13}
}

@article{qian2025principled,
  title={Principled data augmentation for learning to solve quadratic programming problems},
  author={Qian, Chendi and Morris, Christopher},
  journal={arXiv preprint arXiv:2506.01728},
  year={2025}
}

@article{alvarez2017machine,
  title={A machine learning-based approximation of strong branching},
  author={Alvarez, Alejandro Marcos and Louveaux, Quentin and Wehenkel, Louis},
  journal={INFORMS Journal on Computing},
  volume={29},
  number={1},
  pages={185--195},
  year={2017},
  publisher={INFORMS}
}

@inproceedings{zarpellon2021parameterizing,
  title={Parameterizing branch-and-bound search trees to learn branching policies},
  author={Zarpellon, Giulia and Jo, Jason and Lodi, Andrea and Bengio, Yoshua},
  booktitle={Proceedings of the aaai conference on artificial intelligence},
  volume={35},
  number={5},
  pages={3931--3939},
  year={2021}
}

@article{lin2022learning,
  title={Learning to branch with tree-aware branching transformers},
  author={Lin, Jiacheng and Zhu, Jialin and Wang, Huangang and Zhang, Tao},
  journal={Knowledge-Based Systems},
  volume={252},
  pages={109455},
  year={2022},
  publisher={Elsevier}
}

@article{scavuzzo2024machine,
  title={Machine learning augmented branch and bound for mixed integer linear programming},
  author={Scavuzzo, Lara and Aardal, Karen and Lodi, Andrea and Yorke-Smith, Neil},
  journal={Mathematical Programming},
  pages={1--44},
  year={2024},
  publisher={Springer}
}

@article{pan2022deepopf,
  title={DeepOPF: A feasibility-optimized deep neural network approach for AC optimal power flow problems},
  author={Pan, Xiang and Chen, Minghua and Zhao, Tianyu and Low, Steven H},
  journal={IEEE Systems Journal},
  volume={17},
  number={1},
  pages={673--683},
  year={2022},
  publisher={IEEE}
}

@inproceedings{tang2024learning,
title={Learning to Optimize for Mixed-Integer Non-linear Programming with Feasibility Guarantees},
author={Bo Tang and Elias Boutros Khalil and Jan Drgona},
booktitle={Workshop on Differentiable Learning of Combinatorial Algorithms},
year={2025}
}

@inproceedings{fioretto2020predicting,
  title={Predicting ac optimal power flows: Combining deep learning and lagrangian dual methods},
  author={Fioretto, Ferdinando and Mak, Terrence WK and Van Hentenryck, Pascal},
  booktitle={Proceedings of the AAAI conference on artificial intelligence},
  volume={34},
  number={01},
  pages={630--637},
  year={2020}
}

@article{kotary2024learning,
  title={Learning constrained optimization with deep augmented lagrangian methods},
  author={Kotary, James and Fioretto, Ferdinando},
  journal={arXiv preprint arXiv:2403.03454},
  year={2024}
}

@inproceedings{fioretto2020lagrangian,
  title={Lagrangian duality for constrained deep learning},
  author={Fioretto, Ferdinando and Van Hentenryck, Pascal and Mak, Terrence WK and Tran, Cuong and Baldo, Federico and Lombardi, Michele},
  booktitle={Joint European conference on machine learning and knowledge discovery in databases},
  pages={118--135},
  year={2020},
  organization={Springer}
}

@article{qian2025towards,
  title={Towards graph neural networks for provably solving convex optimization problems},
  author={Qian, Chendi and Morris, Christopher},
  journal={arXiv preprint arXiv:2502.02446},
  year={2025}
}

@inproceedings{
rashwan2025enforcing,
title={Enforcing convex constraints in Graph Neural Networks},
author={Ahmed Rashwan and Keith Briggs and Chris Budd and Lisa Maria Kreusser},
booktitle={The Thirty-ninth Annual Conference on Neural Information Processing Systems},
year={2026}
}

@article{qian2026gnnsdp,
      title={On the Expressive Power of GNNs to Solve Linear SDPs}, 
      author={Chendi Qian and Christopher Morris},
      journal={arXiv preprint arXiv:2604.27786},
      year={2026}
}

\clearpage


\appendix

\section{Extended Background}\label{sec:extended_background}

\subsection{Higher-Order Weisfeiler-Leman}
\label{sec:more_wl}

Let $n \ge 1$ and $[n] \coloneqq \{1, \dotsc, n\}$. We denote multisets with $\dgal{\dots}$. A labeled graph $G = (V, E, l)$ consists of a finite node set $V$, an edge set $E$, and a labeling function $l \colon V \to \mathbb{N}$. The neighborhood of a node $v \in V$ is defined as $N(v) \coloneqq \{u \in V \mid (u,v) \in E\}$. Two graphs $G$ and $H$ are isomorphic ($G \simeq H$) if there exists an adjacency- and label-preserving bijection between their node sets. For $k$-tuples $\vec{v}, \vec{w} \in V^k$, we say they share the same atomic type, denoted $\mathsf{atp}(\vec{v}) = \mathsf{atp}(\vec{w})$, if the coordinate mapping $v_i \mapsto w_i$ induces a partial isomorphism.

The 1-Weisfeiler-Leman ($\WLk{1}$) algorithm, also known as color refinement, is a standard heuristic for testing graph isomorphism~\citep{Wei+1968}. It determines if two graphs are non-isomorphic by iteratively updating node colors based on local neighborhoods. Initialized with $\vec{c}^0_v \coloneq l(v)$, the color of node $v$ at iteration $t > 0$ is updated via an injective hash function:
\begin{equation*}
\vec{c}^{t}_v \coloneq \mathsf{hash} \mleft( \vec{c}^{t-1}_v, \dgal[\big]{ \vec{c}^{t-1}_u \mid u \in N(v) } \mright).
\end{equation*}

The algorithm terminates when the partition of colors stabilizes, i.e., $\vec{c}^t_v = \vec{c}^t_w \iff \vec{c}^{t+1}_v = \vec{c}^{t+1}_w$ for all $v, w \in V$, yielding a \new{stable coloring} $\vec{c}^{\infty}_v$. If two graphs yield a differing number of nodes for any stable color, $\WLk{1}$ successfully distinguishes them~\citep{Cai+1992}.

To distinguish more complex symmetries where $\WLk{1}$ fails, the hierarchy generalizes to $k$-tuples where $k \ge 2$, yielding the $k$-dimensional WL ($\WLk{k}$) and its folklore variant ($\FWLk{k}$)~\citep{Bab1979, Imm+1990, Cai+1992, Gro2017}. Both algorithms initialize tuple colors using their atomic types: $\vec{c}^0_{\vec{v}} \coloneq \mathsf{atp}(\vec{v})$. They differ primarily in how they aggregate information when substituting a node $w \in V$ into the $j$-th position of $\vec{v}$, denoted as $\phi_j(\vec{v},w) \coloneq (v_1, \dots, v_{j-1}, w, v_{j+1}, \dots, v_k)$.

The $\FWLk{k}$ algorithm captures the \emph{joint} configuration of these substitutions across all tuple positions simultaneously:
\begin{equation*}
\vec{c}^{t}_{\vec{v}} \coloneq \mathsf{hash} \mleft( \vec{c}^{t-1}_{\vec{v}}, \dgal[\Big]{ \left( \vec{c}^{t-1}_{\phi_1(\vec{v},w)}, \dots,  \vec{c}^{t-1}_{\phi_k(\vec{v},w)} \right) \mid w \in V } \mright).
\end{equation*}
In contrast, standard $\WLk{k}$ aggregates the multiset of each substitution position \emph{independently}:
\begin{equation*}
\vec{c}^{t}_{\vec{v}} \coloneq \mathsf{hash} \mleft( \vec{c}^{t-1}_{\vec{v}}, \left( \dgal{ \vec{c}^{t-1}_{\phi_1(\vec{v},w)} \mid w \in V }, \dots, \dgal{ \vec{c}^{t-1}_{\phi_k(\vec{v},w)} \mid w \in V } \right) \mright).
\end{equation*}
Both algorithms iterate until the tuple colors stabilize.

We quantify expressive power using refinement relations. An algorithm $\mathcal{A}$ \emph{refines} $\mathcal{B}$, denoted $\mathcal{A} \sqsubseteq \mathcal{B}$, if $\mathcal{A}$ can distinguish all pairs of graphs that $\mathcal{B}$ can distinguish. Strict refinement is denoted as $\mathcal{A} \sqsubset \mathcal{B}$. Because $\FWLk{k}$ aggregates joint neighborhood structures rather than marginal distributions, it is strictly more expressive than $\WLk{k}$ for $k \ge 2$. Specifically, $\FWLk{k}$ is proven to be equivalent in power to $\WLk{(k\mathrm{+}1)}$~\citep{Cai+1992, Gro2017}, establishing the following strict expressivity hierarchy:
$$ \dots \equiv \FWLk{(k+1)} \sqsubset \WLk{(k+1)} \equiv \FWLk{k} \sqsubset \WLk{k} \equiv \FWLk{(k\mathrm{-}1)} \sqsubset \dots $$

\subsection{Branch-and-Bound Method}\label{sec:bab_algo}
The branch-and-bound method is a tree-search algorithm for solving combinatorial optimization problems exactly.

As shown in~\cref{alg:bab}, the solver starts with the SDP relaxation of the Max-Cut Problem at the B\&B root node, and repeatedly selects an active node from the queue to create subproblems. By default, we select the node with the highest upper bound to explore. Each subproblem is generated by branching on an unfixed variable in the current graph. In our implementation, the branching variable is chosen by the most-fractional rule. The branching variable is then fixed to the two possible partitions, producing two child nodes. 

The subproblem associated with each child node is relaxed and evaluated as an SDP. This returns an upper bound together with a primal SDP solution. The Goemans--Williamson algorithm then rounds this primal solution to a feasible cut. Specifically, the primal solution is factorized as $X=VV^T$ with row vectors $v_i$. Random vectors $r$ are sampled to induce a partition, where the variables with $\{i | v_i^\top r \geq 0\}$ can be assigned to one partition and the remaining variables are assigned to the other. The feasible cut is then further improved by a local search that flips individual variables to increase the lower bound. If the new cut is better, the incumbent will be updated. Since the SDP relaxation provides an upper bound on the best-cut value achievable by the current subproblem, this child node can be safely pruned whenever its SDP optimal objective value is no better than the incumbent lower bound. The pruning condition is $\mathrm{UB}_c < \mathrm{LB} + 1$, as all edge weights are integers. In this case, neither this node nor any of its descendants can yield a better feasible cut. Otherwise, the node remains active and is added back to the queue.   

In this work, we consider three branch-and-bound solvers: (i) a vanilla SDP-based B\&B solver; (ii) a hybrid B\&B solver; and (iii) a neural B\&B solver. While the vanilla solver uses Mosek to solve the SDP relaxation at the node-evaluation step, where $\mathrm{UB}_c = f^*_{\mathrm{SDP}}$, the hybrid and neural solvers apply the proposed GNN to predict a dual-feasible solution and generate a valid upper bound such that $f^*_{\mathrm{SDP}} \leq f_{\mathrm{GNN}}$. Therefore, the same pruning condition can be applied directly if $f_{\mathrm{GNN}} < \mathrm{LB} + 1$. The hybrid solver uses the GNN for early fathoming. At each node, the GNN is queried first to produce $f_{\mathrm{GNN}}$. If this bound already prunes the node, the exact SDP evaluation with Mosek will be bypassed. Otherwise, Mosek is invoked to compute the tighter bound, $f_{\mathrm{SDP}}$, and the algorithm proceeds exactly as in the vanilla solver. As a result, the hybrid solver preserves the same branch-and-bound tree, but changes the cost of node evaluation. The nodes pruned by GNN avoid solving the SDP with Mosek, whereas surviving nodes incur the cost of both GNN inference and the subsequent Mosek solve. In contrast, the neural solver differs more significantly from the vanilla solver. The Mosek solve is removed entirely from this solver, and the GNN is responsible for the prediction of both primal and dual solutions. Specifically, the neural solver considers $\mathrm{UB}_c = f_{\mathrm{GNN}}$ and uses the predicted primal solution in the Goemans--Williamson algorithm. These changes can alter node priorities in the queue and branching decisions, resulting in a different branch-and-bound tree. Despite this different search trajectory, the neural solver can explore nodes much more efficiently by combining lightweight GNN inference with batched evaluation of child nodes.

\begin{algorithm}[t]
\caption{branch and bound for Max-Cut}
\label{alg:bab}
\begin{algorithmic}[1]
\Require Problem $P$, relaxation method $\mathcal{R} \in \{\textproc{Vanilla}, \textproc{Hybrid}, \textproc{Neural}\}$, top $K$ nodes
\State $Q \gets \emptyset$ \Comment{Queue}
\State $(\mathrm{UB}_r,x_r) \gets \textproc{EvaluateNode}(P,\mathcal{R})$ \Comment{Root upper bound}
\State $x_r^{\mathrm{heur}} \gets \textproc{GW-Algorithm}(x_r)$
\State $x_{\mathrm{LB}} \gets x_r^{\mathrm{heur}}$, $\mathrm{LB} \gets x_r^{\mathrm{heur}} L_{r} x_r^{\mathrm{heur}}$ \Comment{Best solution and lower bound}
\State $\textproc{Push}(Q,r)$
\While{$Q \neq \emptyset$}
    \State $\mathcal{N} \gets \textproc{PopMany}(Q,K)$ \Comment{Best (top K, if $\mathcal{R}=\textproc{Neural}$) upper bound first}
    \State $\mathcal{C} \gets [\ ]$
    \ForAll{$n \in \mathcal{N}$}
        \If{$\mathrm{LB} + 1 \leq \mathrm{UB}_n$}
            \State $j^* \gets \textproc{MostFractional}(n)$
            \State $(c^{\mathrm{down}}, c^{\mathrm{up}}) \gets \textproc{Branch}(n,j^*)$
            \State $\mathcal{C} \gets \textproc{Append}(\mathcal{C}, c^{\mathrm{down}}, c^{\mathrm{up}})$
        \EndIf
    \EndFor

    \ForAll{$c \in \mathcal{C}$}
        \State $(\mathrm{UB}_c,x_c) \gets \textproc{EvaluateNode} (c,\mathcal{R})$ \Comment{Parallel if $\mathcal{R}=\textproc{Neural}$; sequential otherwise}
        \State $x_c^{\mathrm{heur}} \gets \textproc{GW-Algorithm}(x_c)$ \Comment{Parallel if $\mathcal{R}=\textproc{Neural}$; sequential otherwise}
        \If{$x_c^{\mathrm{heur}} L_{c} x_c^{\mathrm{heur}} > \mathrm{LB}$}
            \State $x_{\mathrm{LB}} \gets x_c^{\mathrm{heur}}$, $\mathrm{LB} \gets x_c^{\mathrm{heur}} L_{c} x_c^{\mathrm{heur}}$ \Comment{Update lower bound}
        \EndIf

        \If{$\mathrm{UB}_c < \mathrm{LB} + 1$}
            \State $\textproc{PruneNode}(c)$ \Comment{Prune by optimality}
        \Else
            \State $\textproc{Push}(Q, c)$
        \EndIf
    \EndFor
    
\EndWhile
\State \Return $(x_{\mathrm{LB}},\mathrm{LB})$
\end{algorithmic}
\end{algorithm}

\section{Omitted proofs}
\label{sec:proof}

In the following, we outline proofs omitted from the main paper.

\subsection{Proof of \cref{thm:sparsewl}}
\begin{theorem}[Restatement of \cref{thm:sparsewl}]
Let $\vec{X}^* \in \mathbb{S}^n_+$ and $\vec{y}^*$ denote the optimal primal and dual solutions, respectively, for a given instance of the Max-Cut SDP relaxation. For any indices $(i, j)$ and $(p, q)$, if the stable colors produced by $\dmcwl$ satisfy $\vec{v}^{\infty}_{ij} = \vec{v}^{\infty}_{pq}$, then the corresponding primal entries satisfy $X^*_{ij} = X^*_{pq}$. Furthermore, if the diagonal stable colors satisfy $\vec{v}^{\infty}_{kk} = \vec{v}^{\infty}_{ll}$, then the dual solution satisfies $y^*_k = y^*_l$.
\end{theorem}

We first prove a lemma that our $\dmcwl$ refines the spectral property of a symmetric matrix, similar to $\vctfwl$ in prior work \citep{qian2026gnnsdp}, but different in derivation details.

\begin{lemma}
\label{lem:multiset_spectral}
Let $\vec{M} \in \mathbb{S}^{n}$ be a symmetric matrix with spectral decomposition 
\begin{equation*}
\vec{M} = \sum_{k = 1}^m \lambda_k \vec{P}_k
\end{equation*}
where $\lambda_1 > \lambda_2 \cdots > \lambda_m$ are the non repeating eigenvalues, and $\vec{P}_k$ are Frobenius covariants \citep[p.~437]{horn1994topics}, i.e., the symmetric matrices describing the projection onto the eigenspace of eigenvalue $\lambda_k$. 

The stable coloring $\vec{v}^\infty$ are produced by the multiset $\dmcwl$ algorithm:
\begin{align*}
\vec{v}_{ij}^0 &\coloneq \mathsf{hash} \left( M_{ij}, \mathbb{I}_{i=j} \right) \\
\vec{v}_{ij}^t &\coloneq \mathsf{hash} \Bigl( \vec{v}_{ij}^{t-1},\dgal[\big]{ (\vec{v}_{iu}^{t-1}, M_{uj}) \mid u \in [n], M_{uj}\neq 0} \oplus \dgal[\big]{ (\vec{v}_{uj}^{t-1}, M_{iu}) \mid u \in [n], M_{iu}\neq 0} \Bigr)
\end{align*}
where $\mathbb{I}_{i=j}$ is a diagonal indicator which takes value 1 if $i=j$ otherwise 0. 
The following result holds:
\begin{equation*}
\vec{v}_{ij}^{\infty} = \vec{v}_{pq}^{\infty} \implies (\vec{P}_k)_{ij} = (\vec{P}_k)_{pq}, \quad \text{ for all } k \in [m].
\end{equation*}
\end{lemma}

\begin{proof}
The proof is based on the fact that the spectral projectors $\vec{P}_k$ are polynomials in $\vec{M}$ via Sylvester's formula. We say a coloring $\vec{v}$ refines a matrix $\vec{A}$, defined as:
\begin{equation*}
    \vec{v}_{ij} = \vec{v}_{pq} \implies A_{ij} = A_{pq}.
\end{equation*}
We proceed by induction to demonstrate that the multiset aggregation effectively simulates matrix multiplication: if the $\dmcwl$ coloring at an iteration refines the entries of $\vec{M}^d$, the next iteration refines $\vec{M}^{d+1}$. Consequently, the stable coloring refines any polynomial of $\vec{M}$, and by extension, naturally refines all the specific polynomials defining the spectral projectors.

At initialization $t=0$, the coloring of $\vec{v}^0$ refines $\vec{M}^1$ naturally, as they are injectively defined by $\vec{M}$. Besides, $\vec{v}^0$ also refines $\vec{M}^0 = \vec{I}$, because of the diagonal indicator $\mathbb{I}_{i=j}$. 

Assume that at iteration $t$, the coloring $\vec{v}^{t}$ refines the matrix power $\vec{M}^d$:
\begin{equation*}
\vec{v}^t_{ij} = \vec{v}^t_{pq} \implies \left( \vec{M}^d \right)_{ij} = \left( \vec{M}^d \right)_{pq}, \text{ for all } i,j,p,q.
\end{equation*}
We show that the next iteration $\vec{v}^{t+1}$ will refine $\vec{M}^{d+1}$:
\begin{equation*}
    \vec{v}^{t+1}_{ij} = \vec{v}^{t+1}_{pq} \implies \left( \vec{M}^{d+1} \right)_{ij} = \left( \vec{M}^{d+1} \right)_{pq}.
\end{equation*}
Specifically, we have the following derivation for $\dmcwl$:
\begin{equation*}
\begin{aligned}
& \vec{v}^{t+1}_{ij} = \vec{v}^{t+1}_{pq} \\
\implies &  \dgal[\big]{ (\vec{v}_{iu}^{t}, M_{uj}) \mid u \in [n], M_{uj} \neq 0} \oplus \dgal[\big]{ (\vec{v}_{uj}^{t}, M_{iu}) \mid u \in [n], M_{iu} \neq 0}  = \\
& \dgal[\big]{ (\vec{v}_{pu}^{t}, M_{uq}) \mid u \in [n], M_{uq} \neq 0} \oplus \dgal[\big]{ (\vec{v}_{uq}^{t}, M_{pu}) \mid u \in [n], M_{pu} \neq 0} && \text{ (i)} \\
\implies &  \dgal[\big]{ ((\vec{M}^d)_{iu}, M_{uj}) \mid u \in [n], M_{uj} \neq 0} \oplus \dgal[\big]{ ((\vec{M}^d)_{uj}, M_{iu}) \mid u \in [n], M_{iu} \neq 0}  = \\
& \dgal[\big]{ ((\vec{M}^d)_{pu}, M_{uq}) \mid u \in [n], M_{uq} \neq 0} \oplus \dgal[\big]{ ((\vec{M}^d)_{uq}, M_{pu}) \mid u \in [n], M_{pu} \neq 0} && \text{ (ii)} \\
\implies &  \dgal[\big]{ (\vec{M}^d)_{iu} M_{uj} \mid u \in [n], M_{uj} \neq 0} \oplus \dgal[\big]{ (\vec{M}^d)_{uj} M_{iu} \mid u \in [n], M_{iu} \neq 0}  = \\
& \dgal[\big]{ (\vec{M}^d)_{pu} M_{uq} \mid u \in [n], M_{uq} \neq 0} \oplus \dgal[\big]{ (\vec{M}^d)_{uq} M_{pu} \mid u \in [n], M_{pu} \neq 0} && \text{ (iii)} \\
\implies & \left(\sum_{ u \colon M_{uj} \neq 0} (\vec{M}^d)_{iu} M_{uj}\right) \oplus \left(\sum_{ u  \colon M_{iu} \neq 0} (\vec{M}^d)_{uj} M_{iu} \right) \\
= & \left(\sum_{ u \colon M_{uq} \neq 0} (\vec{M}^d)_{pu} M_{uq}\right) \oplus \left(\sum_{ u \colon M_{pu} \neq 0} (\vec{M}^d)_{uq} M_{pu}\right) && \text{ (iv)}\\
\implies & \left(\sum_{ u \colon M_{uj} \neq 0} (\vec{M}^d)_{iu} M_{uj} + \sum_{ u \colon M_{uj} = 0} (\vec{M}^d)_{iu} M_{uj}\right) \\
\oplus & \left(\sum_{ u  \colon M_{iu} \neq 0} (\vec{M}^d)_{uj} M_{iu} + \sum_{ u  \colon M_{iu} = 0} (\vec{M}^d)_{uj} M_{iu} \right) \\
= & \left(\sum_{ u \colon M_{uq} \neq 0} (\vec{M}^d)_{pu} M_{uq} + \sum_{ u \colon M_{uq} = 0} (\vec{M}^d)_{pu} M_{uq}\right) \\
\oplus & \left(\sum_{ u \colon M_{pu} \neq 0} (\vec{M}^d)_{uq} M_{pu} + \sum_{ u \colon M_{pu} = 0} (\vec{M}^d)_{uq} M_{pu}\right) && \text{ (v)}\\
\implies & \left(\sum_{ u } (\vec{M}^d)_{iu} M_{uj}\right) \oplus \left(\sum_{ u  } (\vec{M}^d)_{uj} M_{iu} \right) \\
= & \left(\sum_{ u } (\vec{M}^d)_{pu} M_{uq}\right) \oplus \left(\sum_{ u } (\vec{M}^d)_{uq} M_{pu}\right) \\
\implies &  (\vec{M}^{d} \vec{M})_{ij} \oplus (\vec{M} \vec{M}^{d})_{ij}  = (\vec{M}^{d} \vec{M})_{pq} \oplus (\vec{M} \vec{M}^{d})_{pq} \\
\implies &  (\vec{M}^{d+1})_{ij}  = (\vec{M}^{d+1})_{pq}. \\
\end{aligned}
\end{equation*}

Among the derivations, (i) holds because it is a partial definition; (i) to (ii) holds because $\vec{v}^t$ refines $\vec{M}^d$ by hypothesis; (ii) to (iii) holds because of injectivity, if the injective function of $\mathsf{hash}(a, b) = \mathsf{hash}(c, d)$, then $\mathsf{hash}(a \cdot b) = \mathsf{hash}(c \cdot d)$; (iii) to (iv) also holds because of injectivity; if two multisets of real values are equal, then the sum aggregation of their elements must also be equal; (iv) to (v) is straightforward because it is adding zeros to (iv). 

Therefore, $\vec{v}^{t+1}_{ij} = \vec{v}^{t+1}_{pq} \implies \left( \vec{M}^{d+1} \right)_{ij} = \left( \vec{M}^{d+1} \right)_{pq}$. 

By induction, the stable coloring $\vec{v}^\infty$ refines $\vec{M}^d$ for all $d \in \mathbb{N}$. Consequently, it also refines any matrix polynomial $p(\vec{M}) = \sum_{d=0}^\infty a_d \vec{M}^d$. This holds because the expressivity to distinguish the structural features of individual powers $\vec{M}^d$ implies the ability to distinguish their linear combinations.
The real-valued matrix $\vec{M}$ is symmetric and diagonalizable, therefore we apply spectral decomposition $\vec{M} = \sum_{k=1}^m \lambda_k \vec{P}_k$, where matrices $\vec{P}_k$ are given by the Sylvester's formula \citep[p.~437]{horn1994topics}: 
$$ \vec{P}_k = \prod_{j \neq k} \frac{\vec{M} - \lambda_j \vec{I}}{\lambda_k - \lambda_j}. $$
Because each $\vec{P}_k$ is a polynomial in $\vec{M}$, the stable coloring $\vec{v}^\infty$ refines each $\vec{P}_k$. Thus, it holds for all $k$ that $\vec{v}_{ij}^{\infty} = \vec{v}_{pq}^{\infty} \implies (\vec{P}_k)_{ij} = (\vec{P}_k)_{pq}$.
\end{proof}

Next, we adapt the proof of $\vctfwl$ to fit $\dmcwl$, using the PDHG algorithm as a proof technique. As SDPs can have multiple solutions, which complicates analysis, we also adopt the assumption of \citet{qian2026gnnsdp} that the SDP has a unique solution with the minimum Frobenius norm. 

Our work for Max-Cut relaxations is a special case of SDPs, as the $\Acal$ is a diagonal operation. We define the operator:
\begin{equation*}
\text{diag} \colon \mathbb{S}^n \to \mathbb{R}^n, \quad \text{diag}(\vec{X}) = (X_{11}, X_{22}, \dots, X_{nn})\trans
\end{equation*}
and its adjoint:
\begin{equation*}
\text{Diag} \colon \mathbb{R}^n \to \mathbb{S}^n, \quad \text{Diag}(\vec{y}) = \begin{bmatrix}
    y_1 & 0 & \cdots & \\
    0 & y_2 & \cdots & \\
    & & \ddots & 0 \\
    & \cdots & 0 & y_n \\ 
\end{bmatrix}.
\end{equation*}

The PDHG updates for the Frobenius-regularized SDP are given by \citet{qian2026gnnsdp}:
\begin{equation}
\label{eq:pdhg_proof}
\begin{aligned}
\vec{X}^{t+1} & \coloneq \Proj_{\mathbb{S}^n_+} \left[ \dfrac{\vec{X}^t - \alpha_t \text{Diag}(\vec{y}^t) - \alpha_t \vec{C}}{1 + \alpha_t \varepsilon} \right] \\
\vec{y}^{t+1} & \coloneq \vec{y}^t + \beta_t \text{diag} \left( \vec{X}^{t+1} + \theta_t \left( \vec{X}^{t+1} - \vec{X}^t \right) \right) - \beta_t \vec{b}.
\end{aligned}
\end{equation}
The $\Proj$ function projects a symmetric matrix $ \vec{X} \in \mathbb{S}^n $ onto the PSD cone by finding the closest PSD matrix $ \vec{X}_{{\mathbb{S}}^n_+} $ that minimizes the Frobenius norm distance $ \lVert \vec{X}_{{\mathbb{S}}^n_+} - \vec{X} \rVert_F $. Given the spectral decomposition of the input matrix $ \vec{X} = \sum_{i=1}^n \lambda_i \vec{v}_i \vec{v}_i^\top $, where $ \lambda_i $ are the real eigenvalues and $ \vec{v}_i $ are the corresponding orthonormal eigenvectors. The projected matrix is then reconstructed by truncating any negative eigenvalues to zero: $ \vec{X}_{{\mathbb{S}}^n_+} = \sum_{i=1}^n \max(\lambda_i,0) \cdot \vec{v}_i \vec{v}_i^\top $.

For any $\varepsilon > 0$, the SDP problem has a unique primal solution $\vec{X}^*$. And if $\varepsilon \rightarrow 0$, it reduces to normal linear SDP. We initialize with $\vec{X}^0 = \bm{0}_{n \times n}$ and $\vec{y}^0 = \bm{0}_m$.

Now we proceed with our proof.

\begin{proof}[Proof of \cref{thm:sparsewl}]
We provide a constructive proof by simulating the primal-dual hybrid gradient (PDHG) algorithm \citep{chambolle2016ergodic,wang2024tuningfree}. We show that the stable coloring of $\dmcwl$ refines every update of the PDHG algorithm, which is proven to converge in \citet{qian2026gnnsdp,wang2024tuningfree}. 

Our goal is to show that for all $t \ge 0$, the stable $\dmcwl$ coloring on variable and constraint nodes refines the intermediate primal-dual solutions of PDHG updates:
\begin{align*}
\vec{v}^{\infty}_{ij} = \vec{v}^{\infty}_{pq} & \implies X^t_{ij} = X^t_{pq} \\
\vec{v}^{\infty}_{kk} = \vec{v}^{\infty}_{ll} & \implies y^t_{k} = y^t_{l}.
\end{align*}

Before the inductive step, we establish three key facts derived from the properties of the stable coloring.

\xhdr{Fact 1} Since the stable coloring $\vec{v}^\infty$ refines the initial coloring $\vec{v}^0$, and $\vec{v}^0$ is defined by $\vec{C}$, it holds that $\vec{v}^{\infty}_{ij} = \vec{v}^{\infty}_{pq} \implies C_{ij} = C_{pq}$. Furthermore, by applying \Cref{lem:multiset_spectral} to $\vec{C}$, if its spectral decomposition is given by $\vec{C} = \sum_h \lambda_h \vec{P}_h$, then the projection matrices also respect the coloring: $\forall h \colon \vec{v}^{\infty}_{ij} = \vec{v}^{\infty}_{pq} \implies (\vec{P}_h)_{ij} = (\vec{P}_h)_{pq}$. Besides, we notice that the constraints are uniform for Max-Cut, i.e., $\vec{b} = \vec{e}$, therefore, it does not contribute to separation power. 

\xhdr{Fact 2} We notice that the \textbf{diagonal} and \textbf{off-diagonal} elements receive a different set of colors at the initialization, due to the diagonal indicator $\mathbb{I}$, and their color sets remain disjoint throughout the update of $\dmcwl$.

\xhdr{Proof by induction}
We now prove the theorem by induction on $t$.
At initialization $t=0$, we have $\vec{X}^0 = \bm{0}$ and $\vec{y}^0 = \bm{0}$. The implications $\vec{v}^{\infty}_{ij} = \vec{v}^{\infty}_{pq} \implies X^0_{ij} = X^0_{pq}$ and $\vec{v}_{kk}^{\infty} = \vec{v}_{ll}^{\infty} \implies y_k^0 = y_l^0$ hold trivially.

Suppose the hypothesis holds for iteration $t \geq 0$. We show it holds for iteration $t+1$. 

\textbf{1. For diagonal ones}

During the primal update
\begin{equation*}
\vec{X}^{t+1} \coloneq \Proj_{\mathbb{S}^n_+} \left[ \dfrac{\vec{X}^t - \alpha_t \text{Diag}(\vec{y}^t) - \alpha_t \vec{C}}{1 + \alpha_t \varepsilon} \right],
\end{equation*}
consider the term inside the PSD projection in:
\begin{equation*}
\vec{Z}^t \coloneq  \dfrac{\vec{X}^t - \alpha_t \text{Diag}(\vec{y}^t) - \alpha_t \vec{C}}{1 + \alpha_t \varepsilon},
\end{equation*}
since $\vec{X}^t$ and $\text{Diag}(\vec{y}^t)$ (by hypothesis), $\vec{C}$ (by Fact 1), are all refined by the coloring $\vec{v}^\infty$, their linear combination $\vec{Z}^t$ also is. 
With Fact 2, one can easily verify this by inspecting the diagonal and off-diagonal elements separately. For off-diagonals $i \neq j$ and $p \neq q$, if $\vec{v}^\infty_{ij} = \vec{v}^\infty_{pq}$, then $\left(\vec{X}^t - \alpha_t \vec{C}\right)_{ij} = \left(\vec{X}^t - \alpha_t \vec{C}\right)_{pq}$. And for diagonals, if $\vec{v}^\infty_{kk} = \vec{v}^\infty_{ll}$, then $\left(\vec{X}^t - \alpha_t \text{Diag}(\vec{y}^t) - \alpha_t \vec{C} \right)_{kk} = \left(\vec{X}^t - \alpha_t \text{Diag}(\vec{y}^t) - \alpha_t \vec{C} \right)_{ll}$. 

Then, the update is $\vec{X}^{t+1} = \Proj_{\mathbb{S}^n_+} (\vec{Z}^t)$. The PSD projection depends only on the spectral decomposition. Specifically, spectral decomposition is performed on $\vec{Z}^t$, and the Frobenius covariants corresponding to negative eigenvalues are removed. By \Cref{lem:multiset_spectral}, we know $\forall (i,j,p,q) \colon \vec{v}^{\infty}_{ij} = \vec{v}^{\infty}_{pq} \implies X^{t+1}_{ij} = X^{t+1}_{pq}$, that is, $\vec{v}^{\infty}$ refines $\vec{X}^{t+1}$. 

Finally, in the dual update step, we evaluate the term inside the operator $\text{diag}$:
\begin{equation*}
\vec{W}^{t+1} \coloneq \vec{X}^{t+1} + \theta_t \left( \vec{X}^{t+1} - \vec{X}^t \right).
\end{equation*}
Since both $\vec{X}^{t+1}$ and $\vec{X}^t$ are refined by $\vec{v}^\infty$, so is $\vec{W}^{t+1}$.
Therefore, $\text{diag}(\vec{W}^{t+1})$ is naturally refined by $\left(\vec{v}^\infty_{11}, \cdots, \vec{v}^\infty_{nn} \right)\trans$. 
Since $\vec{y}^t$ (by hypothesis) and $\vec{b}$ (by Fact 1) are both refined by the coloring $\left(\vec{v}^\infty_{11}, \cdots, \vec{v}^\infty_{nn} \right)\trans$, the sum $\vec{y}^{t+1}$ is refined by the coloring.

In summary, we have proven by induction that if the stable coloring $\vec{v}^\infty$ refines $\vec{X}^{t}$ and $\vec{y}^t$ in the PDHG algorithm, they also refine $\vec{X}^{t+1}$ and $\vec{y}^{t+1}$. 

Since PDHG is guaranteed to converge to $\vec{X}^*$ and $\vec{y}^*$, this completes the proof.
\end{proof}

\subsection{Proof of Neuralization}
\label{sec:proof_neural}
\begin{proposition}
There exists a set of parameters for the functions $\mathsf{INIT}$, $\mathsf{UPD}$, $\mathsf{MSG}$, and $\mathsf{MAP}$ such that the neural architectures $\mcmpnn$ and $\dmpnn$ possess maximal expressivity equivalent to their corresponding algorithms.
\end{proposition}

The proof follows the universal approximation arguments for multiset functions established in \citet{qian2026gnnsdp}.

\begin{proof}
To prove that the neural architectures achieve the exact expressive power of their corresponding $\mathsf{WL}$ algorithms, we rely on the universal approximation of injective multiset functions. Following \citet{qian2026gnnsdp} and \citet{Mar+2019}, let $\mathcal{X}$ be a space of bounded multisets. \emph{There exists a continuous, injective function parameterized by an MLP that maps any bounded multiset to a unique real-valued vector representation.}

Since our architectures are deterministic unrollings of the $\mcwl$ and $\dmcwl$ algorithms, they cannot distinguish variables that the WL tests cannot. Thus, we only need to show that the neural networks are at most as expressive as the $\mathsf{WL}$ injective hash functions, meaning that the neural network layers can injectively simulate them.

$\mcmpnn$ can be derived from $\vctfwl$ \citep{qian2026gnnsdp}, the aggregation simply replaces the standard bipartite, variable-to-constraint message passing with variable-to-variable topology. The sum over $u \in [n]$ combined with the MLPs $\mathsf{MSG}$ and $\mathsf{MAP}$ injectively encodes the neighborhood multisets, identical to the variable-node update proof in \citet{qian2026gnnsdp}. 

For the sparse variant $\dmpnn$, the algorithm aggregates over conditionally filtered multisets based on the sparse topology of $\vec{C}$. The neural formulation handles this by explicitly masking the messages:
$$ \sum_{\substack{u \in [n] \\ C_{iu} \neq 0}} \mathsf{MSG}^{t} \bigl( \vec{h}_{uj}^{t-1}, C_{iu} \bigr) $$
Because the matrix entries $C_{iu}$ belong to a compact space, the $\mathsf{MSG}$ function can be parameterized to injectively map the pair $(\vec{h}_{uj}^{t-1}, C_{iu})$ to a latent vector. The summation over the non-zero indices perfectly implements the multiset encoding of the sparse neighbors. Finally, the update function $\mathsf{UPD}^t$ concatenates the outputs of these independent multiset summations and injectively maps them to the next layer's feature space. 

Therefore, there exists a set of weights for which the neural architectures perfectly simulate the injective hash functions of their $\mathsf{WL}$ algorithms, establishing equivalent expressive power.
\end{proof}

\section{Details of experiment design}
\label{sec:exp_design}
\paragraph{Comparison of GNN variants on uniform dataset} 
All GNN variants are set to 10 layers with a hidden dimension of 128. The models are trained using the Adam optimizer \citep{Kin+2015} with default hyperparameters and batch size of 128. Experiments were conducted on a compute cluster equipped with four NVIDIA L40S GPUs. We train for a maximum of $\num{1000}$ epochs, using a scheduler that decays the learning rate if the averaged dual objective on the validation set fails to improve for 100 consecutive epochs. To prevent overfitting, early stopping is triggered after $\num{200}$ epochs without validation improvement. We report the results averaged over 3 random initialization seeds.

Despite the sparsity of the underlying graphs, we implement $\dmpnn$ using dense matrix multiplications, as modern GPUs are heavily optimized for dense operations, and we find they process them significantly faster than batched sparse-matrix alternatives.

\paragraph{B\&B experiments on Max-Cut dataset}
The GNN used in B\&B experiments are set to 6 layers with a hidden dimension of 96. With a batch size of $\num{32}$, We first train the GNN with dual heads for $\num{1000}$ epochs under the same protocol, then freeze these parameters and train the primal heads for another $\num{200}$ epochs with patience $\num{100}$.
In the~\cref{sec:results}, each benchmark instance family contains ten different instances. We solve all instances to global optimality without a timeout and report the average number of evaluated nodes and solving time.

\section{Additional experiments}
\label{sec:add_exps}

\subsection{Ablation: Choice of SDP solver}\label{sec:alationSDPsolver}
To justify our selection of Mosek \citep{mosek} as our primary classical solver, we conduct an empirical performance comparison across the Max-Cut instances from \citet{biqmac,biqbin}. Specifically, we evaluate the computational time required to solve the SDP relaxation at the root node using three distinct optimization paradigms: SDPLR, a non-linear approach based on the Burer-Monteiro method \citep{burer2003nonlinear}; Mosek \citep{mosek}, a state-of-the-art interior-point method; and SCS \citep{o2021operator}, a first-order solver utilizing the ADMM algorithm \citep{wen2010alternating}. The experiments are conducted on the AMD EPYC 7662 CPU. The solving times are summarized in \cref{tab:solver_results}.

\begin{table*}[htb!]
\caption{Comparison of mean solving times (in seconds) for the root-node SDP relaxation across various Max-Cut instance families.}
\label{tab:solver_results}
\centering
\resizebox{0.5\textwidth}{!}{
\begin{tabular}{lccc}
\toprule
Instance Family & Mosek & SCS & SDPLR \\
\midrule
g05\_60 & 0.0146 & 0.1040 & 0.1423 \\
g05\_80 & 0.0235 & 0.2135 & 0.2023 \\
g05\_100 & 0.0380 & 0.4145 & 0.3018 \\
g05\_180 & 0.1208 & 2.0893 & 1.0229 \\
pm1s\_80 & 0.0221 & 0.1760 & 0.1521 \\
pm1s\_100 & 0.0347 & 0.3192 & 0.1792 \\
pm1s\_180 & 0.1135 & 1.7188 & 0.4433 \\
pw01\_100 & 0.0376 & 0.7706 & 0.1894 \\
pw01\_180 & 0.1158 & 3.7741 & 0.5261 \\
pw05\_100 & 0.0399 & 0.9922 & 0.3982 \\
pw05\_180 & 0.1268 & 7.0904 & 1.4823 \\
pw09\_100 & 0.0408 & 1.1123 & 0.3519 \\
pw09\_180 & 0.1307 & 6.4214 & 1.1196 \\
w01\_100 & 0.0373 & 0.6187 & 0.1872 \\
w01\_180 & 0.1179 & 9.5305 & 0.5013 \\
w05\_100 & 0.0384 & 1.0688 & 0.3590 \\
w05\_180 & 0.1230 & 4.9640 & 1.3138 \\
w09\_100 & 0.0393 & 0.9056 & 0.3445 \\
w09\_180 & 0.1265 & 4.7360 & 1.0651 \\
\bottomrule
\end{tabular}
}
\end{table*}

As demonstrated in \cref{tab:solver_results}, Mosek consistently outperforms both the first-order and Burer-Monteiro methods by a substantial margin across all evaluated instance sizes and densities. This superior computational efficiency strongly motivates our exclusive use of Mosek for exact SDP evaluations throughout our branch-and-bound framework. 

\subsection{Ablation: Number of layers}

We identify a key structural limitation of $\dmpnn$ at shallow depths. As shown in \cref{tab:layers}, $\dmpnn$ yields a significantly higher objective gap than $\mcmpnn$ when utilizing only 4 layers. Although this gap becomes marginal in deeper networks, $\mcmpnn$ consistently performs better. This difference intuitively stems from how each architecture expands its structural receptive field. Because $\mcmpnn$ jointly aggregates pairs of previous embeddings, it simulates matrix multiplication. At layer $t$, it effectively refines $\vec{C}^{2^t}$. Conversely, $\dmpnn$ aggregates features with the static matrix $\vec{C}$, restricting its refinement to $\vec{C}^{t+1}$ at layer $t$. Therefore, while $\dmpnn$ remains theoretically valuable, the rapid field expansion of $\mcmpnn$ makes it more favorable when network depth is limited.

\begin{table*}[htb!]
\caption{Relative objective gap (\%) under varying network depths on the $N=100, p=0.1$ ER graph dataset.}
\label{tab:layers}
\centering
\resizebox{0.6\textwidth}{!}{
\begin{tabular}{lccc}
\toprule
\multirow{2}{*}{Model} & \multicolumn{3}{c}{\textbf{\# Layers}} \\
 & 4 & 6 & 10 \\
\midrule
$\mcmpnn$ 
& \textbf{0.931$\pm${\scriptsize 0.001}} 
& \textbf{0.656$\pm${\scriptsize 0.002}} 
& \textbf{0.598$\pm${\scriptsize 0.005}} \\
$\dmpnn$  
& 1.403$\pm${\scriptsize 0.004} 
& 0.872$\pm${\scriptsize 0.002} 
& 0.628$\pm${\scriptsize 0.004} \\ 
\bottomrule
\end{tabular}
}
\end{table*}

\subsection{Ablation: OOD size generalization}

To evaluate the out-of-distribution (OOD) size generalization of our architectures, we test models trained on smaller ER graph instances with $N=100$ against progressively larger graphs ($N=150$ and $N=200$) while maintaining the edge probability $p=0.1$. As detailed in \cref{tab:sizegen}, both $\mcmpnn$ and $\dmpnn$ exhibit a steady increase in the relative objective gap as the graph size scales up, highlighting the inherent difficulty of zero-shot size extrapolation. While both architectures degrade on larger instances, $\mcmpnn$ consistently shows a lower objective gap across all tested sizes, confirming that it is better equipped to generalize to unseen, larger-scale topologies.

\begin{table*}[htb!]
\caption{Size generalization performance (relative objective gap \%) on larger graphs. Models were evaluated on $N \in \{150, 200\}$ with fixed edge probability $p=0.1$.}
\label{tab:sizegen}
\centering
\resizebox{0.5\textwidth}{!}{
\begin{tabular}{lcc}
\toprule
{Dataset} & $\mcmpnn$ & $\dmpnn$ \\
\midrule
N=100, p=0.1
& 0.598$\pm${\scriptsize 0.005} 
& 0.628$\pm${\scriptsize 0.004} \\
N=150, p=0.1
& 2.545$\pm${\scriptsize 0.476} 
& 3.201$\pm${\scriptsize 0.547} \\
N=200, p=0.1
& 4.089$\pm${\scriptsize 0.453} 
& 4.712$\pm${\scriptsize 0.466} \\
\bottomrule
\end{tabular}
}
\end{table*}

\subsection{Ablation: Variants of normalization and GNNs}
\label{sec:ab_norm}

To further understand the architectural sensitivities of our model, we conducted an ablation study evaluating the impact of different normalization strategies. 

As mentioned in the \citet{qian2026gnnsdp}, they add normalization layers between GNN updates to stabilize training. For a mini-batch of $B$ graphs, variable embeddings are zero-padded to form a uniform tensor of shape $B \times n_{\max} \times n_{\max} \times d$, where $n_{\max}$ is the maximum graph size in the batch and $d$ is the feature dimension. Because the update MLPs contain biases, padded entries can distort normalization statistics. To investigate this, we formalize and ablate three normalization strategies:

The \textsf{SpatialNorm} layer normalizes each feature channel $f \in [d]$ across the entire padded spatial grid. The mean $\mu_f$ and variance $\sigma_f^2$ are computed as:
\begin{equation}
\mu_f = \frac{1}{n_{max}^2} \sum_{i=1}^{n_{max}} \sum_{j=1}^{n_{max}} h_{ij, f},  \quad \sigma_f^2 = \frac{1}{n_{max}^2} \sum_{i=1}^{n_{max}} \sum_{j=1}^{n_{max}} (h_{ij, f} - \mu_f)^2.
\end{equation}
Features are then standardized and shifted via learnable parameters $\gamma_f, \beta_f \in \mathbb{R}$:
\begin{equation}
\hat{h}_{ij, f} = \gamma_f \left( \frac{h_{ij, f} - \mu_f}{\sqrt{\sigma_f^2 + \epsilon}} \right) + \beta_f,
\end{equation}
where $\epsilon$ is a numerical stability constant. Crucially, because statistics are aggregated over the fixed $n_{\max} \times n_{\max}$ grid, the presence of zero-padding for graphs where the true size $n < n_{\max}$ inherently skews the normalization. 

As a baseline, we use \textsf{LayerNorm} \citep{ba2016layer}, which operates independently on each variable node $\vec{h}_{ij} \in \mathbb{R}^d$. It computes the mean $\mu_{ij}$ and variance $\sigma_{ij}^2$ exclusively across the feature dimension $d$:
\begin{equation}
\mu_{ij} = \frac{1}{d} \sum_{f=1}^{d} h_{ij, f}, \quad \sigma_{ij}^2 = \frac{1}{d} \sum_{f=1}^{d} (h_{ij, f} - \mu_{ij})^2.
\end{equation}
The vector is then scaled using learnable parameters $\vec{\gamma}, \vec{\beta} \in \mathbb{R}^d$:
\begin{equation}
\hat{\vec{h}}_{ij} = \vec{\gamma} \odot \left( \frac{\vec{h}_{ij} - \mu_{ij}\mathbf{1}}{\sqrt{\sigma_{ij}^2 + \epsilon}} \right) + \vec{\beta},
\end{equation}
Since it is performed per-node, \textsf{LayerNorm} is topology-agnostic. Consequently, it is completely immune to the zero-padding artifacts that degrade \textsf{SpatialNorm}.

We also evaluate \textsf{GraphNorm} \citep{cai2021graphnorm}. Unlike \textsf{SpatialNorm}, \textsf{GraphNorm} dynamically flattens the tensor to strictly isolate valid (unpadded) nodes. Let $\mathcal{V}_g = \{ (i,j) \mid i, j \in [n]\}$ denote the actual node indices for a graph $g$ of size $n$. \textsf{GraphNorm} computes statistics for each channel $f$ exclusively over this valid set:
\begin{equation}
\mu_{g, f} = \frac{1}{|\mathcal{V}_g|} \sum_{(i,j) \in \mathcal{V}_g} h_{ij, f}, \quad \sigma_{g, f}^2 = \frac{1}{|\mathcal{V}_g|} \sum_{(i,j) \in \mathcal{V}_g} (h_{ij, f} - \mu_{g, f})^2.
\end{equation}
Valid features are then normalized using a learnable mean-scaling parameter $\alpha_f$ alongside $\gamma_f, \beta_f \in \mathbb{R}$:
\begin{equation}
\hat{h}_{ij, f} = \gamma_f \left( \frac{h_{ij, f} - \alpha_f \mu_{g, f}}{\sqrt{\sigma_{g, f}^2 + \epsilon}} \right) + \beta_f.
\end{equation}
By explicitly masking padded entries, \textsf{GraphNorm} calculates the true statistics of each graph, completely isolating it from padding-induced distortions.

We trained the network on the Branch-and-Bound dataset with root nodes of ER graphs ($N=100, p=0.5$) across various graph sizes using \textsf{LayerNorm}, \textsf{GraphNorm}, \textsf{SpatialNorm}, and an unnormalized baseline (\textsf{NoNorm}), analyzing both the absolute and relative dual objective gaps on the test set. As illustrated in \cref{fig:norms}, the techniques exhibit two fundamentally different scaling behaviors. \textsf{SpatialNorm} yields highly suboptimal performance on small graphs (e.g., fewer than 50 nodes) because the excessive zero-padding heavily distorts its spatial statistics. Consequently, it shows a distinct downward trend in error as the node count increases and the padding ratio drops. In contrast, \textsf{LayerNorm}, \textsf{GraphNorm}, and \textsf{NoNorm} display the exact opposite trajectory. All three achieve exceptionally low objective gaps on small graphs but exhibit a steady, monotonic increase in error as the topology expands. Interestingly, the surprisingly competitive performance of \textsf{NoNorm} demonstrates that our proposed architecture is inherently well-conditioned, effectively learning the underlying SDP mapping without strictly requiring intermediate normalization to stabilize training. Ultimately, despite these divergent initial trajectories, all configurations tightly converge to a similar performance baseline as the graph size approaches 100 nodes and padding vanishes.

\begin{figure}
    \centering
    \includegraphics[width=\textwidth]{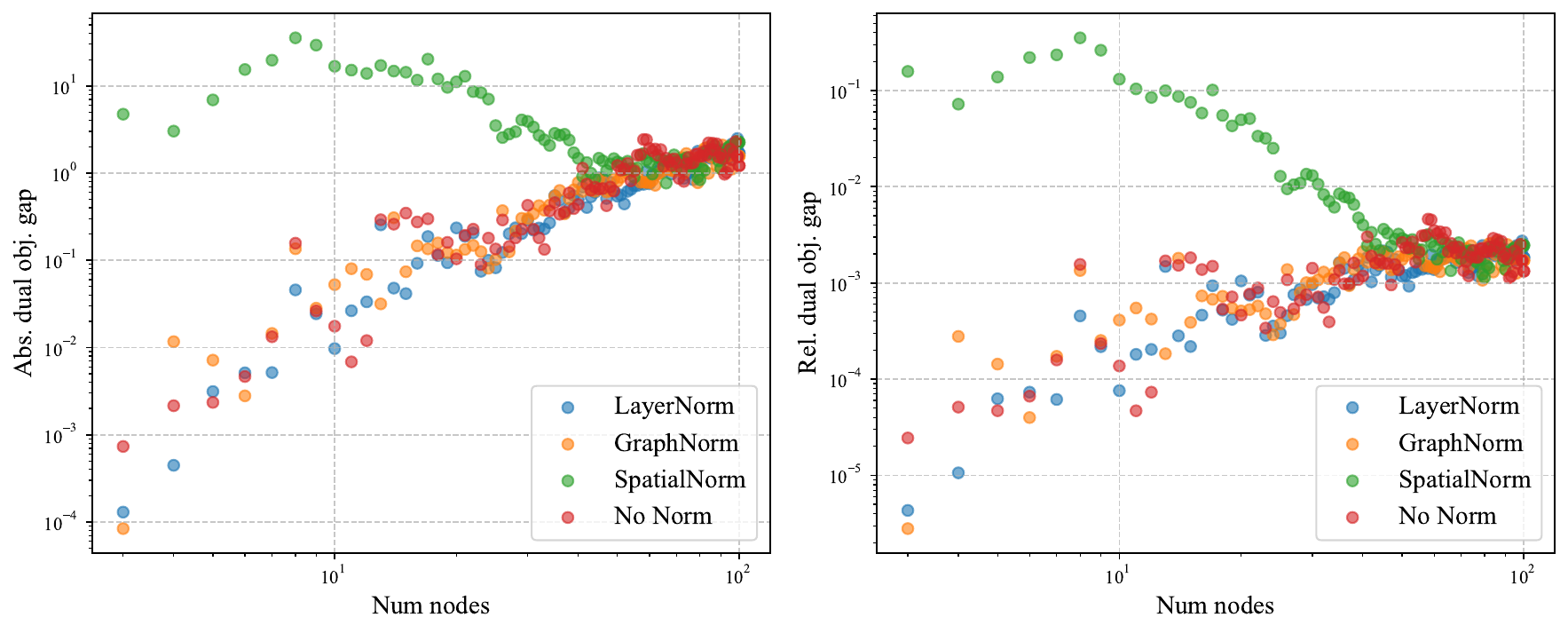} 
    \caption{Effect of normalization layers on objective gaps. The absolute (left) and relative (right) dual objective gaps across different numbers of nodes in the test set are shown in the plots. }
    \label{fig:norms}
\end{figure}

To complement our theoretical analysis, we empirically contrast the $\mcmpnn$ and $\dmpnn$ variants across different normalization layers. After training on the Branch-and-Bound dataset, we assess their generalization on test instances of varying sizes. As shown in \cref{fig:sparse_norms}, the scaling trend remains highly consistent between the two architectures. However, $\dmpnn$ demonstrates comparatively weaker convergence on larger graphs, resulting in worse performance than its $\mcmpnn$ counterpart, which is consistent with our ablation of the number of layers. 

\begin{figure}
    \centering
    \includegraphics[width=\textwidth]{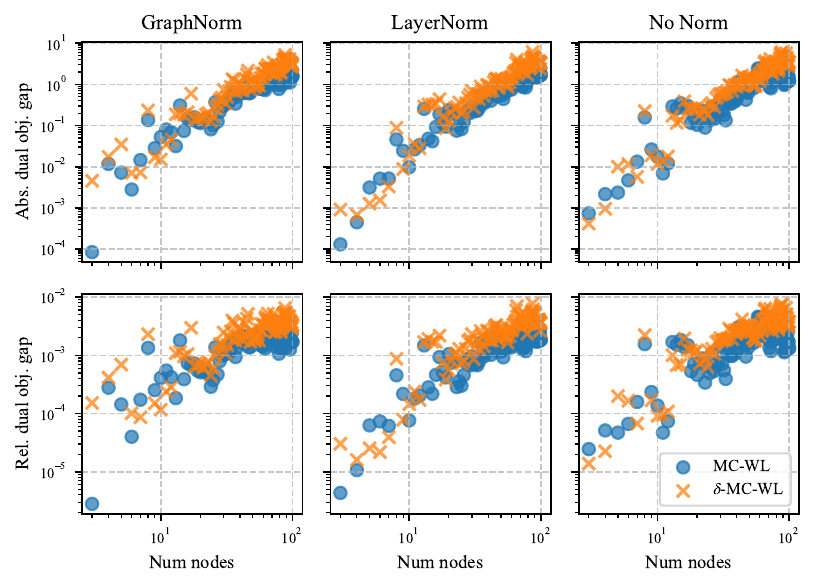} 
    \caption{Performance of $\mcmpnn$ and $\dmpnn$ under different normalizations. Absolute (top) and relative (bottom) dual objective gaps evaluated on the test set.}
    \label{fig:sparse_norms}
\end{figure}

We further tested the network with \textsf{LayerNorm}, \textsf{GraphNorm}, and \textsf{NoNorm} , as well as $\mcmpnn$ and $\dmpnn$ in established Max-Cut datasets \citep{biqmac,biqbin}. Instance families cover a diverse range of graph sizes, densities, and edge weights: (1) \texttt{g05\_n} (edge probability 0.5) for $n \in \{60,80,100,180\}$ with edge weights in $\{0, 1\}$; (2) \texttt{pm1s\_n} (edge probability 0.1) for $n \in \{80,100,180\}$ with edge weights in $\{-1, 0, 1\}$; (3) \texttt{w0d\_n} for  $n \in \{100,180\}$ and $d \in \{1,5,9\}$ (edge probability 0.1, 0.5, 0.9) with integer edge weights in $[-10, 10]$; (4) \texttt{pw0d\_n} for  $n \in \{100,180\}$ and $d \in \{1,5,9\}$ (edge probability 0.1, 0.5, 0.9) with integer edge weights in $[0, 10]$. Notably, among these families, only \texttt{g05\_100} aligns with the training distribution. The proposed GNN is trained on graphs with $n=100$ nodes, edge probability 0.5, and edge weights in $\{0, 1\}$, together with contracted subgraphs generated along random branching trajectories. As a result of contraction, these smaller subgraphs are typically denser than $0.5$ and no longer uniformly weighted in $\{0, 1\}$. 

We evaluated the out-of-distribution generalization of the proposed GNN by measuring the objective gap of the SDP relaxation of these instances in~\cref{fig:obj_gap_model}. \textsf{LayerNorm} and \textsf{NoNorm} consistently achieves better objective gap on in-training-distribution instances (\texttt{g05\_100}) and near-training-distribution instances (\texttt{g05\_60}, \texttt{g05\_80}). Similarly, the dense $\mcwl$ variant outperforms the $\dmcwl$ variant on these same families. However, \textsf{LayerNorm} exhibits significantly worse generalization than \textsf{GraphNorm} in out-of-distribution instances (\texttt{pm1s\_n}, \texttt{pw0d\_n}, \texttt{w0d\_n}). In particular, \textsf{LayerNorm} does not generalize when only the graph size differs (\texttt{g05\_180}), while all other configurations can result in reasonably small objective gaps. However, if the GNN with \textsf{LayerNorm} and $\mcwl$ is trained within a distribution where graphs are sparse and edge weights are sampled from a broader range, it still generalizes well to larger graphs as shown in~\cref{fig:obj_gap_dist}. Under \textsf{GraphNorm}, the generalization of $\mcwl$ variant is slightly better when the edge-weight distribution is relatively narrow (\texttt{g05}, \texttt{pm1s}, \texttt{pwd}), and slightly worse when the edge-weight distribution is wider (\texttt{wd}). Across all configurations, the smallest objective gaps are typically achieved in instance families with edge probability 0.5 (\texttt{g05\_n}, \texttt{pw05\_n}, \texttt{w05\_n}), followed by those with edge probability 0.9 (\texttt{pw09\_n}, \texttt{w09\_n}). In contrast, sparse instances with edge probability 0.1 (\texttt{pm1s\_n}, \texttt{pw01\_n}, \texttt{w01\_n}) have the worst objective gaps. This trend aligns with the training distribution, which consists of 100-node $50\%$ density graphs and contracted subgraphs that are usually denser than $50\%$. 

\begin{figure}
    \centering
    \includegraphics[width=\textwidth]{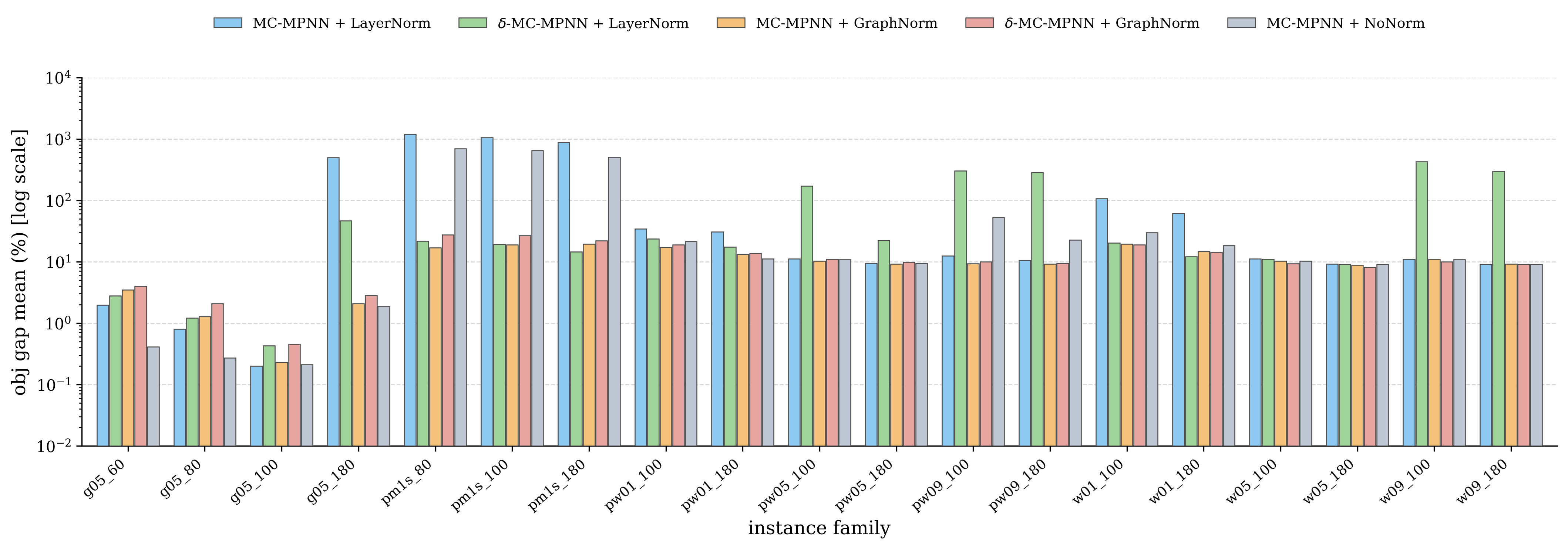} 
    \caption{Effect of normalization layers and layer variants on objective gaps, trained within the distribution of \texttt{g05\_100} and subproblems. The average relative dual objective gaps across different instance families are shown in the plots. }
    \label{fig:obj_gap_model}
\end{figure}

\begin{figure}
    \centering
    \includegraphics[width=\textwidth]{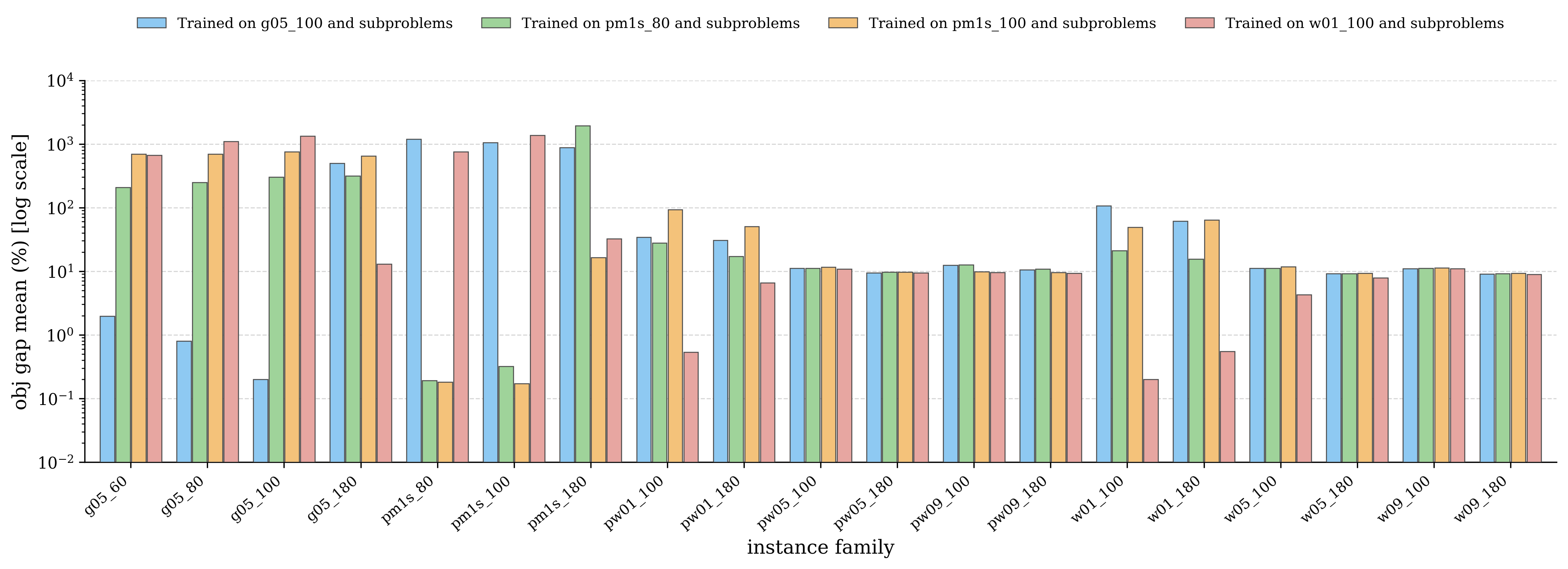} 
    \caption{Effect of training distribution on objective gaps, with \textsf{LayerNorm} and $\mcwl$. The average relative dual objective gaps across different instance families are shown in the plots. }
    \label{fig:obj_gap_dist}
\end{figure}

\subsection{Ablation: Branch-and-bound experiments without Goemans--Williamson algorithm}
\label{sec:ab_heur_free}
We consider B\&B experiments without Goemans--Williamson algorithm to evaluate the impact of learned upper bounds only. By explicitly setting the lower bound to the optimal objective value at the root node, we eliminate the role of primal heuristics in the comparison. The following experiments are performed on the AMD EPYC 7662 CPU and NVIDIA A100 GPU.

\begin{table*}[htb!]
\caption{Branch-and-bound results without Goemans--Williamson algorithm on established instances. The Hybrid and GNN methods utilize the $\mcmpnn$ + $\mathsf{LayerNorm}$ model. The best results across all configurations are highlighted.}
\centering
\label{tab:ab_heur_free}
\resizebox{0.8\textwidth}{!}{
\begin{tabular}{lcccc}
\toprule
Instance Family & Method & Top-K & \# Nodes Evaluated & Solving Time (s) \\
\midrule
\multirow{4}{*}{\texttt{g05\_60}} 
& Vanilla  & -  & 3566 & 36.4 \\ 
& Hybrid & 1  & 3566 & 44.8 \\ 
& Neural    & 1  & 4562 & 42.5 \\ 
& Neural    & 32 & \textbf{4563} & \textbf{15.1} \\ 
\midrule

\multirow{4}{*}{\texttt{g05\_80}} 
& Vanilla  & -  & 49128 & 796.8 \\ 
& Hybrid & 1  & 49128 & 805.0 \\ 
& Neural    & 1  & 53970 & 507.8 \\ 
& Neural    & 32 & \textbf{53975} & \textbf{224.4} \\ 
\midrule

\multirow{4}{*}{\texttt{g05\_100}} 
& Vanilla  & -  & 1087290 & 25580.0 \\ 
& Hybrid & 1  & 1087290 & 26215.0 \\ 
& Neural    & 1  & 1272820 & 12508.0 \\ 
& Neural    & 32 & \textbf{1272807} & \textbf{6116.3} \\ 

\bottomrule
\end{tabular}
}
\end{table*}

\subsection{Ablation: Additional branch-and-bound experiments with Goemans--Williamson algorithm}
\label{sec:ab_additional_bab}
This section reports results of additional branch-and-bound experiments in the standard setting. In particular, we explore and compare the impact of different GNN variants, normalization choices, and top-$K$ selection on the neural solver. The following experiments are performed on the AMD EPYC 7662 CPU and NVIDIA L40S GPU.

As an additional reference, we report results using BiqCrunch \citep{biqcrunch}, a highly optimized state-of-the-art exact solver for Max-Cut. It is crucial to note, however, that the full BiqCrunch algorithm is \textbf{not directly comparable} to our proposed framework. BiqCrunch is a sophisticated branch-and-cut solver whose main performance advantage comes from dynamically strengthening the continuous SDP relaxation with separated triangle inequalities. Rather than adding all such inequalities upfront, it iteratively identifies the most violated ones and maintains a compact working set of useful cuts. The resulting strengthened SDP bounds are computed efficiently using a specialized quasi-Newton method. In contrast, the core algorithmic contribution of our work is a neural surrogate designed strictly to accelerate the evaluation of the base SDP relaxation during node enumeration, rather than to replace polyhedral cutting-plane machinery. \textbf{Extending our GNN architecture to explicitly incorporate or predict such triangle inequalities remains a promising direction for future work}.

To establish a fair comparison of the underlying bounding mechanisms, we additionally evaluated BiqCrunch with its triangle inequalities disabled. Under this pure B\&B setting, BiqCrunch becomes significantly less competitive than both our baseline and our neural solver. Without the benefit of the cuts, its quasi-Newton optimizer is less efficient at solving the basic unstrengthened SDP than the interior-point methods utilized by Mosek. Therefore, the vanilla B\&B solver equipped with Mosek remains the most appropriate, direct, and rigorous baseline for evaluating the acceleration provided by our neural SDP surrogate.


Across the near-distribution families \texttt{g05\_60} \cref{tab:bab_g5_60} and \texttt{g05\_80} \cref{tab:bab_g05_80}, as well as the in-distribution \texttt{g05\_100} \cref{tab:bab_g05_100} instances, the combination of the $\mcmpnn$ variant and \textsf{LayerNorm} consistently yields the strongest branch-and-bound performance. Notably, the neural solver successfully generalizes to smaller subproblems without requiring retraining, demonstrating robust inference efficiency. The practical computational advantage becomes increasingly pronounced as the graph size grows, the neural solver achieves a substantial $5.58\times$ speedup on \texttt{g05\_60}, and reaches an impressive $9.1\times$ acceleration over the vanilla solver on the \texttt{g05\_100} instances.

Interestingly, even though \textsf{NoNorm} can achieve a smaller objective gap than \textsf{LayerNorm} at the initial root node see~\cref{fig:obj_gap_model}, \textsf{LayerNorm} performs significantly better once the GNN is dynamically deployed deeper inside the branch-and-bound tree. This implies that \textsf{LayerNorm} is highly effective at keeping the feature distributions stable and closer to the training distribution during subsequent branching steps. Furthermore, the batched comparison confirms that batched node evaluation is another major source of acceleration. However, this speedup does not scale linearly with $K$; while increasing $K$ continues to improve total running time, it yields diminishing returns once $K$ is sufficiently large, e.g., $K=32$.

Finally, for extended evaluations demonstrating the solver's performance on sparser and out-of-distribution instance families, please refer to the results for \texttt{pm1s\_80} \cref{tab:bab_pm1s_80}, \texttt{pm1s\_100} \cref{tab:bab_pm1s_100}, and \texttt{w01\_100} \cref{tab:bab_w01_100} detailed below.

\begin{table*}[htb!]
\caption{Additional branch-and-bound results on \texttt{g05\_60} instance family. The neural network is trained on graphs in distribution of \texttt{g05\_100}. The best results across all neural solver configuration are highlighted.}
\centering
\label{tab:bab_g5_60}
\resizebox{0.9\textwidth}{!}{
\begin{tabular}{cccccc}
\toprule
 Method & Top-K & \# Nodes Evaluated & Solving Time (s) & Speed-up ($\times$) \\
\midrule
 Vanilla & - 
& 3566   %
& 53.1   & 1.00 \\ %
 BiqCrunch (w/o cuts) & - 
& 3918   %
& 153.9   & 0.34 \\ %
 BiqCrunch (w/ cuts) & - 
& 7   %
& 2.3   & 26.50 \\ %
\midrule

\multirow{3}{*}{$\mcmpnn$ + $\mathsf{LayerNorm}$}  & 1
&   4562  %
&   34.2  & 1.55\\ %
 & 8
&    4562 %
&   10.0  & 5.30 \\ %
 & 32
& \textbf{4562}   %
& \textbf{9.5}  & \textbf{5.58} \\ %
\midrule

\multirow{3}{*}{$\dmpnn$ + $\mathsf{LayerNorm}$} & 1
&  4617   %
& 33.6 & 1.58 \\ %
 & 8
&  4618 %
&  10.3 & 5.15 \\ %
 & 32
&  4618  %
&  9.7 & 5.46 \\ %
\midrule

\multirow{3}{*}{$\mcmpnn$ + $\mathsf{GraphNorm}$} & 1
&  4663 
&  42.2 & 1.26 \\
 & 8
&  4663 
&  11.3 & 4.69 \\
 & 32
&   4663
&  10.4 & 5.10\\
\midrule

\multirow{3}{*}{$\mcmpnn$ + $\mathsf{NoNorm}$} & 1
&   5291
& 39.3 & 1.35 \\
 & 8
&   5291
&   11.8 & 4.49 \\
 & 32
&  5291 
&  11.0  & 4.82 \\
\midrule

\multirow{3}{*}{$\mcmpnn$ + $\mathsf{SpatialNorm}$} & 1
&  5211 
&  41.8 & 1.27\\
 & 8
&  6511
&   14.7 & 3.61\\
 & 32
&   9600
&  19.3 & 2.75 \\
\bottomrule
\end{tabular}
}
\end{table*}

\begin{table*}[htb!]
\caption{Additional branch-and-bound results on \texttt{g05\_80} instance family. The neural network is trained on graphs in the distribution of \texttt{g05\_100}. The best results across all neural solver configurations are highlighted.}
\centering
\label{tab:bab_g05_80}
\resizebox{0.9\textwidth}{!}{
\begin{tabular}{cccccc}
\toprule
 Method & Top-K & \# Nodes Evaluated & Solving Time (s) & Speed-up ($\times$) \\
\midrule
 Vanilla & - 
& 49128    %
& 1170.3  & 1.00  \\ %
 BiqCrunch (w/o cuts) & -
& 48633   %
& 3746.3 & 0.31    \\ %
 BiqCrunch (w/ cuts) & -
& 49   %
& 22.0 & 53.20    \\ %
\midrule

\multirow{3}{*}{$\mcmpnn$ + $\mathsf{LayerNorm}$} & 1
&   53969  %
&   444.9 & 2.63  \\ %
 & 8
&   53973  %
&   156.3 & 7.49  \\ %
 & 32
&  \textbf{53975}  %
&  \textbf{155.7} & \textbf{7.52}  \\ %
\midrule

 \multirow{3}{*}{$\dmpnn$ + $\mathsf{LayerNorm}$} &1
&  64437   %
&  494.7 & 2.37   \\ %
 &8
&  64461   %
&   175.2 & 6.68  \\ %
 &32
&  64474  %
&   179.8 & 6.51 \\ %
\midrule

 \multirow{3}{*}{$\mcmpnn$ + $\mathsf{GraphNorm}$} &1
&  61250   %
&  595.7 & 1.96  \\ %
 &8
&   61243  %
&   197.1 & 5.94  \\ %
 &32
& 61251   %
&  203.3 & 5.76  \\ %
\midrule

 \multirow{3}{*}{$\mcmpnn$ + $\mathsf{NoNorm}$} &1
&  65593 
&  581.1 & 2.01 \\
 &8
&   65596
&   190.5 & 6.14 \\
 & 32
&  65596 
&  189.0 & 6.19 \\
\midrule

 \multirow{3}{*}{$\mcmpnn$ + $\mathsf{SpatialNorm}$} & 1
&  59094 
&  521.8 & 2.24 \\
 & 8
&  77691 
&  213.4 & 5.48  \\
 & 32
&  120523 
&  314.3 & 3.72 \\
\bottomrule
\end{tabular}
}
\end{table*}

\begin{table*}[htb!]
\caption{Additional branch-and-bound results on \texttt{g05\_100} instance family. The neural network is trained on graphs in the distribution of \texttt{g05\_100}. The best results across all neural solver configurations are highlighted.}
\centering
\label{tab:bab_g05_100}
\resizebox{0.9\textwidth}{!}{
\begin{tabular}{cccccc}
\toprule
 Method & Top-K & \# Nodes Evaluated & Solving Time (s) & Speed-up ($\times$) \\
\midrule
 Vanilla & - 
& 1087290    %
& 40930.1   & 1.00 \\ %
 BiqCrunch (w/ cuts) & -
& 390   %
& 270.7 & 151.20   \\ %
\midrule

 \multirow{2}{*}{$\mcmpnn$ + $\mathsf{LayerNorm}$} &1
& 1272822    %
& 11791.0 & 3.47   \\ %
 &32
&  \textbf{1272801}  %
&  \textbf{4509.3} & \textbf{9.08}  \\ %
\midrule

 $\dmpnn$ + $\mathsf{LayerNorm}$ &32
&   1660692 %
&    5716.0 & 7.16 \\ %
 $\mcmpnn$ + $\mathsf{GraphNorm}$ &32
&  1326781
&  5438.9 & 7.53  \\
\bottomrule
\end{tabular}
}
\end{table*}

\begin{table*}[htb!]
\caption{Additional branch-and-bound results on \texttt{pm1s\_80} instance family. The neural network is trained on graphs in the distribution of \texttt{pm1s\_80}. The best results across all neural solver configurations are highlighted.}
\centering
\label{tab:bab_pm1s_80}
\resizebox{0.9\textwidth}{!}{
\begin{tabular}{cccccc}
\toprule
 Method & Top-K & \# Nodes Evaluated & Solving Time (s) & Speed-up ($\times$) \\
\midrule

 Vanilla & - 
&  25656   %
&   609.5  & 1.00 \\ %
 BiqCrunch (w/ cuts) & -
& 2   %
& 1.2 & 507.92  \\ %
\midrule

 \multirow{2}{*}{$\mcmpnn$ + $\mathsf{LayerNorm}$} & 1
&  27809   %
&  212.5 & 2.87  \\ %
 & 32
&   \textbf{27809} %
&    \textbf{80.8} & \textbf{7.54} \\ %
\bottomrule
\end{tabular}
}
\end{table*}

\begin{table*}[htb!]
\caption{Additional branch-and-bound results on \texttt{pm1s\_100} instance family. The neural network is trained on graphs in the distribution of \texttt{pm1s\_100}. The best results across all neural solver configurations are highlighted.}
\centering
\label{tab:bab_pm1s_100}
\resizebox{0.9\textwidth}{!}{
\begin{tabular}{cccccc}
\toprule
 Method & Top-K & \# Nodes Evaluated & Solving Time (s) & Speed-up ($\times$) \\
\midrule

 Vanilla & - 
& 254408    %
& 10382.5  & 1.00  \\ %
 BiqCrunch (w/ cuts) & -
& 13   %
& 11.7  & 887.39  \\ %
\midrule

\multirow{2}{*}{$\mcmpnn$ + $\mathsf{LayerNorm}$} & 1
&  302311  %
&  2518.5 & 4.12  \\ %
& 32
&  \textbf{302314}  %
&  \textbf{1058.2} & \textbf{9.81} \\ %
\bottomrule
\end{tabular}
}
\end{table*}

\begin{table*}[htb!]
\caption{Additional branch-and-bound results on \texttt{w01\_100} instance family. The neural network is trained on graphs in the distribution of \texttt{w01\_100}. The best results across all neural solver configurations are highlighted.}
\centering
\label{tab:bab_w01_100}
\resizebox{0.9\textwidth}{!}{
\begin{tabular}{cccccc}
\toprule
 Method & Top-K & \# Nodes Evaluated & Solving Time (s) & Speed-up ($\times$) \\
\midrule

 Vanilla & - 
& 269407    %
& 11348.6  & 1.00  \\ %
 BiqCrunch (w/ cuts) & -
& 4   %
& 4.1 & 2767.95   \\ %
\midrule

\multirow{2}{*}{$\mcmpnn$ + $\mathsf{LayerNorm}$} & 1
&  305716  %
&  2689.8 & 4.22 \\ %
& 32
& \textbf{305720}   %
&  \textbf{1073.9} & \textbf{10.57}  \\ %
\bottomrule
\end{tabular}
}
\end{table*}



\end{document}